\documentclass[nohyperref]{article}

\usepackage{microtype}
\usepackage{graphicx}
\usepackage{subfigure}
\usepackage{booktabs} % for professional tables

\usepackage{hyperref}

\usepackage[accepted]{icml2022}

\usepackage{amsmath}
\usepackage{amssymb}
\usepackage{mathtools}
\usepackage{amsthm}
\usepackage{bbm}
\usepackage[capitalize,noabbrev]{cleveref}
\def\Real{\mathop{\mathbb{R}}\nolimits}

\def\argmin{\mathop{\rm argmin}\nolimits}

\newcommand{\bd}{\boldsymbol{d}}

\newcommand{\bs}{\boldsymbol{s}}

\newcommand{\bu}{\boldsymbol{u}}

\newcommand{\bx}{\boldsymbol{x}}
\newcommand{\by}{\boldsymbol{y}}

\newcommand{\bU}{\boldsymbol{U}}
\newcommand{\bV}{\boldsymbol{V}}

\newcommand{\bX}{\boldsymbol{X}}

\newcommand{\btheta}{\boldsymbol{\theta}}

\newcommand{\bTheta}{\boldsymbol{\Theta}}

\newcommand{\E}{\mathbb{E}}
\newcommand{\F}{\mathcal{F}}

\newcommand{\hth}{\widehat{\bTheta}_n}

\usepackage{bbm}
\newcommand{\C}{\mathscr{C}}
 \newcommand{\sP}{\mathbbm{P}}
 \newcommand{\tth}{\bTheta_\ast}
\usepackage{mathrsfs}
\newcommand{\one}{\mathbbm{1}}
\theoremstyle{plain}
\newtheorem{theorem}{Theorem}[section]
\newtheorem{proposition}[theorem]{Proposition}
\newtheorem{lemma}[theorem]{Lemma}

\theoremstyle{definition}
\newtheorem{definition}[theorem]{Definition}
\newtheorem{assumption}{A}
\theoremstyle{remark}

\usepackage{color}

\usepackage[textsize=tiny]{todonotes}

\icmltitlerunning{Bregman Power $k$-Means}

\begin{document}

\twocolumn[
\icmltitle{Bregman Power \textit{k}-Means for Clustering Exponential Family Data}




\begin{icmlauthorlist}
\icmlauthor{Adithya Vellal}{duke}
\icmlauthor{Saptarshi Chakraborty}{berk}
\icmlauthor{Jason Xu}{duke}
\end{icmlauthorlist}

\icmlaffiliation{duke}{Department of Statistical Science, Duke University, Durham, NC, USA.}
\icmlaffiliation{berk}{Department of Statistics, University of California, Berkeley, CA, USA}

\icmlcorrespondingauthor{Jason Xu}{jason.q.xu@duke.edu}

\icmlkeywords{Machine Learning, ICML}

\vskip 0.3in
]



\printAffiliationsAndNotice{\icmlEqualContribution} % otherwise use the standard text.

\begin{abstract}
Recent progress in center-based clustering algorithms combats poor local minima by implicit annealing, using a family of generalized means. These methods are variations of Lloyd's celebrated $k$-means algorithm, and are most appropriate for spherical clusters such as those arising from Gaussian data. In this paper, we bridge these algorithmic advances to classical work on hard clustering under Bregman divergences, which enjoy a bijection to exponential family distributions and are thus well-suited for clustering objects arising from a breadth of data generating mechanisms. The elegant properties of Bregman divergences allow us to maintain closed form updates in a simple and transparent algorithm, and moreover lead to new theoretical arguments for establishing finite sample bounds that relax the bounded support assumption made in the existing state of the art. Additionally, we consider thorough empirical analyses on simulated experiments and a case study on rainfall data, finding that the proposed method outperforms existing peer methods in a variety of non-Gaussian data settings.
\end{abstract}
 \section{Introduction and Background}
%  \color{blue}
%  Saptarshi: flesh out a rough introduction and can copy and paste relevant parts/background from other papers; I will edit and polish
 
%  Example organization: survey recent advances in extending $k$-means and mention how they address many weaknesses, but still are catered to or implicitly assume Gaussian data. Then talk about other divergences and Bregman being one that naturally connects to exponential family distributions, making it more suitable for clustering counts etc. 
 
%  We can cite the uniform concentration bounds paper as one that also suggests objectives under Bregman divergences enjoy nice properties as well as robustness with outliers, but they did not do any empirical or algorithmic/imeplementation work. Instead of the generic adagrad based method in that framework, here we make use of the mean as minimizer property of Bregman divergences to admit a simple and elegant algorithm, and show that it outperforms for exponential family data not from Gaussian. We also inherit strong theoretical guarantees, and we formulate those results for power $k$-means under Bregman divergences.
%  \color{black}
Clustering, the task of finding naturally occurring groups within a dataset, is a cornerstone of the unsupervised learning paradigm. %finding applications across many fields, such as biology, social sciences, physics etc.  
Among a vast literature on clustering algorithms, center-based methods remain widely popular, and $k$-means \citep{macqueen1967some,lloyd1982least} remains the most prominent example 60 years after its introduction. Given $n$ data points $\mathcal{X}=\{\bX_i:i =1,\dots,n\}\subset \Real^p$,  $k$-means seeks to partition the data into $k$ groups in a way that minimizes the within-cluster variance. Representing the cluster centroids $\bTheta = \{\btheta_1,\dots,\btheta_k\} \subset \Real^p$ and  for some dissimilarity measure $d(\cdot,\cdot)$, $k$-means is formulated as the minimization of the objective function
\begin{equation}\label{obj1}
 f_{k\text{-means}}(\bTheta) = \sum_{i=1}^n \min_{1 \le j \le k} d(\bX_i , \btheta_j). 
\end{equation}  Taking the squared Euclidean distance $d(\bx,\by) = \|\bx-\by\|_2^2$ yields the classical $k$-means formulation, while Bregman hard clustering \citep{banerjee2005clustering} allows $d$ to be any Bregman divergence. % which we review ein section \ref{background}. 

 Unfortunately, $k$-means and its variants suffer from well-documented shortcomings such as  sensitivity to initial guess \cite{vassilvitskii2006k,bachem2017distributed,deshpande2020robust}, stopping at poor local minima \cite{zhang1999k,xu2019power}, and fragility to outliers \cite{paul2021uniform} that continue to be addressed in recent work.  %and in high-dimensional settings \cite{witten2010framework,chakraborty2020entropy}. 
 In particular, a drawback we seek to address in this article is the implicit assumption behind $k$-means that the data can be clustered spherically, which works well in Gaussian settings but can fail to separate even simple data examples otherwise \cite{ng2002spectral}. To ameliorate this issue, researchers have proposed various dissimilarity measures \cite{banerjee2005clustering,de2012minkowski,chakraborty2017k,brecheteau2021robust} that admit non-elliptical contours. Among these, the choice of Bregman divergences is appealing \cite{telgarsky2012agglomerative,paul2021on} as their many nice mathematical properties are amenable to analysis and effective algorithms. Their connection to exponential families makes them ideal for many common data generating mechanisms.
 
 Like classic $k$-means, analogs such as Bregman hard clustering  are susceptible to  local optima due to non-convexity of the objective. Wrapper methods such as $k$-means ++ \cite{arthur2007k} and its variants alleviate the problem to an extent, though methods continue to struggle as dimension increases  \cite{aggarwal2001surprising}. Recently, \cite{xu2019power,chakraborty2020entropy} tackle this problem by gradually annealing the optimization landscape  in the Euclidean case. Theoretical work by \cite{paul2021uniform} proposes a clustering framework that encompasses Bregman divergences, establishing desirable properties such as robustness, but does not implement or empirically analyze the Bregman case. The authors advocate generic iterative optimization, using adaptive gradient descent  for the general case \cite{JMLR:v12:duchi11a}. 
 
 In this paper, we propose and analyze a scalable, transparent clustering algorithm that performs annealing to target the same objective as Bregman hard clustering. That is, it inherits nice properties and interpretability while being less prone to poor local solutions.  Leveraging the mean-as-minimizer property of Bregman divergences leads to a simple and elegant algorithm with closed form updates through majorization-minimization (MM). % \cite{xu2018majorization}. 
 We show that it outperforms alternatives on a range of exponential family data via thorough simulation studies.
 
 Moreover, we formulate the method so that it inherits a number of strong theoretical guarantees. Through a novel and extensive theoretical study, we bound the excess risk by appealing to the recent literature on sub-exponential concentration inequalities and the classic approach of bounding the Rademacher complexity through Dudley's entropy integral. These include generalization bounds and learning rates for a broad family of distributions, lifting a restrictive condition that assumes the data has bounded support in previous analyses \citep{paul2021uniform}. Our results not only match the best known rates in literature while imposing much weaker assumptions, but also reveal an interesting dependency of the excess risk on the true cluster centers $\tth$ and the second moment of the underlying data distribution.
 
%  The main contributions of this article is summarized as follows:
%  \begin{itemize}
%      \item We propose a framework for data clustering with Bregman divergences under the power $k$-means paradigm.
%      \item The proposed Bregman power $k$-means not only efficiently avoids getting trapped in local minima even in high-dimensional settings, but also admits a scalable optimization procedure, similar to Lloyd's algorithm. 
%      \item Through a novel and extensive theoretical study, we derive generalization bounds on the excess risk by appealing to the recent literature on sub-exponential concentration inequalities and the classic approach of bounding the Rademacher complexity through Dudley's entropy integral. The novelty of this generalization bound is that it derived in terms of the true cluster centroid $\tth$ and the spread of the data in addition to the usual model parameters, clearly showing the dependency on the true underlying distribution $P$, which might not have a bounded support. 
%      \item Through detailed experimental studies, we show the efficacy of the proposal, especially on data generated from exponential families.
%  \end{itemize}

%\section{Background}
%\label{background}
We briefly overview some relevant concepts that will be used in formulating the Bregman power $k$-means method.

\paragraph{Bregman divergences}
A differentiable, convex function $\phi: \Real^p \to \Real$ generates the \textit{Bregman divergence} $d_{\phi}: \Real^p \times \Real^p \to \Real_{\ge 0}$ %($\Real_{\ge 0}$ denoting the set of non-negative reals) 
defined by  \begin{equation}\label{eq:Bregman} d_\phi(\bx,\by) = \phi(\bx) - \phi(\by) - \langle \nabla \phi(\by) , \bx - \by \rangle . \end{equation}
It becomes clear that $d_\phi(\bx, \by) \geq 0 \ \forall \ \bx, \by \in \Real^p$ since $\phi(\bx) \geq \phi(\by) + \langle \nabla \phi(\by) , \bx - \by \rangle$ is synonymous with $\phi$ being a convex function. From a geometric perspective, $d_\phi(\bx, \by)$ can be thought of as the distance between $\phi(\bx)$ and the first-order approximation of $\phi(\bx)$ centered at $\phi(\by)$. In more intuitive terms, this can be described as the distance between $\phi(\bx)$ and the value of the tangent line to $\phi(\by)$ evaluated at the point $\bx$. 
For instance, taking $\phi(\bu)=\|\bu\|_2^2$ generates the Euclidean distance. Without loss of generality, one may assume $\phi(\mathbf{0}) = \nabla \phi(\mathbf{0})=0.$ 

While not necessarily symmetric like the usual Euclidean distance, Bregman divergences satisfy numerous desirable properties which make them useful for quantifying dissimilarity. They are non-negative and maintain linearity; for any convex $f, g$, we have $d_{\alpha f + \beta g} = \alpha d_f + \beta d_g \forall \ \alpha, \beta \in \Real_{> 0}$. Of note is that Bregman divergences obey a mean-as-minimizer property. As shown in \citet{banerjee2005optimality}, this can be characterized in such a way as stated in a proposition we rephrase below:  \begin{proposition}\label{prop:banerjee} Let $d: \Real^p \times \Real^p \to \Real_{\geq 0}$ to be any continuous function with continuous first-order partial derivatives obeying $d(\bx, \bx) = 0$. Then the mean $\mathbb{E}[X]$ serves as the unique minimizer of $\mathbb{E}[d(X,\by)]$ for $\by \in \Real^p$ if and only if there exists some $\phi$ such that $d = d_{\phi}$.
\end{proposition} 
Furthermore, Bregman divergences share a one-to-one correspondence with regular exponential families, detailed in the next section, making them well-suited for learning from many common data types arising from exponential family distributions. Here the notion of  Bregman information $I_\phi(X) = \min_{\bs \in \textit{dom}(X)} \mathbb{E}[d_\phi(X, \bs)]$ provides a natural measure of distortion. This is minimized at $\bs = \mathbb{E}[X]$ (cf. Prop. \ref{prop:banerjee}), and $I_\phi(X)$ can thus be interpreted as a generalization of variance when spread around the mean of $X$ is measured under $d_\phi$.

 % 
% \textbf{Note:} would be a great section for Adithya to take first stab at writing. Just defining Bregman divergence, giving a few perspectives on it and possibly the classic diagram, mentioning some properties such as the mean as minimizer, connection to exponential families. This is largely synthesizing or rewording information for instance from Bregman k means paper and other sources we have read (blog posts, notes, etc).

% I liked this writeup for background, perhaps paraphrase concisely key points, can be longer than the MM and power $k$-means background paragraphs

% http://mark.reid.name/blog/meet-the-Bregman-divergences.html

\paragraph{Majorization-minimization}
The principle of MM has become increasingly popular in optimization and statistical learning \citep{mairal2015incremental,lange2016mm}.
 %Indeed, Lloyd's algorithm can be interpreted as an EM algorithm for a Gaussian mixture model (GMM) with vanishing variances or as a variational EM approximation with isotropic GMMs \cite{lucke2017}, it is natural that the broader MM principle underpins our own method.
Rather than minimizing an objective of interest $f$ directly, an MM algorithm successively minimizes a sequence of simpler \textit{surrogate functions} $g(\btheta \mid \btheta_n)$ that  \textit{majorize} the original objective  $f(\btheta)$ at the current iterate
$\btheta_m$. Majorization is defined by two conditions: tangency  $g(\btheta_m \mid \btheta_m) =  f(\btheta_m)$ at the current iterate, and domination $g(\btheta \mid \btheta_m)  \geq f(\btheta)$ for all $\btheta$. The steps of the MM algorithm are defined by the rule
\vspace{-0.1cm} 
\begin{equation}\label{eq:MMiter}
\btheta_{m+1} := \arg\min_{\btheta}\; g(\btheta \mid \btheta_m), \end{equation}\vspace{-0.2cm}
which immediately implies the descent property
\begin{eqnarray*}
f(\btheta_{m+1}) \, \leq \, g(\btheta_{m+1} \mid \btheta_{m}) 
\, \le \,  g(\btheta_{m} \mid \btheta_{m}) 
\, = \, f(\btheta_{m}). \label{eq:descent}
\end{eqnarray*}
%\vspace{-0.1cm}
That is, a decrease in $g$ results in a decrease in $f$.
Note that $g(\btheta_{m+1} \mid \btheta_{m} ) \le g(\btheta_{m} \mid \btheta_{m})$ does not require $\btheta_{m+1}$ to minimize $g$ exactly, so that any descent step in $g$ suffices. 
%Similarly, maximizing a function can be accomplished analogously via sequential minorization-maximization. 
The MM principle offers a general prescription for transferring a difficult optimization task onto a sequence of simpler problems \citep{LanHunYan2000}, and includes the well-known EM algorithm for maximum likelihood estimation under missing data as a special case \citep{BecYanLan1997}. 
\vspace{-0.1cm}
\paragraph{Power means}
% \vspace{-0.1cm}
% \begin{eqnarray}
%  \sum_{i=1}^n \Big(\frac{1}{k}  \sum_{j=1}^k \|\bx_i-\btheta_j\|^{-2} \Big)^{-1} := \tilde{f}_{-1}(\bTheta ). \label{KHM}
% \end{eqnarray}
% %\vspace{-0.2cm}
 %Though effective, this proxy becomes less appropriate as the dimension increases, and no longer preserves the within-cluster variance interpretation of \eqref{obj1}. 
Power means are a class of generalized means  defined $M_s(\by)=\left(\frac{1}{k}\sum_{i=1}^k y_i^s \right)^{1/s}$ for a vector $\by$. We see that $s>1$ corresponds to the usual $\ell_s$-norm of $\by$, $s=1$ to the arithmetic mean, and $s=-1$ to the harmonic mean. Power means possess a number of nice properties: they are homogeneous, monotonic, and differentiable with 
\vspace{-0.1cm}
\begin{eqnarray}\label{eq:firstpartial}
 \frac{\partial}{\partial y_j}   M_ s( \by) & =&  \Big(\frac{1}{k}\sum_{i=1}^k y_i^s\Big)^{\frac{1}{s}-1} \frac{1}{k}y_j^{s-1} ,\label{power_mean_grad}
\end{eqnarray}
\vspace{-0.1cm}
and importantly they satisfy the limits
\begin{subequations}\label{eq:limit}
\begin{equation}
    \lim_{s \to -\infty}M_s(\by)=\min\{y_1,\ldots,y_k\}
\end{equation} \vspace{-10pt}
\begin{equation}
    \lim_{s \to \infty}M_s(\by)=\max\{y_1,\ldots,y_k\} .
\end{equation}
\end{subequations}
Further, the well-known power mean inequality holds: for any $s \le t$, 
$M_s (\by) \le M_ t (\by)$ %for any $s \le t$ holds 
\citep{steele2004cauchy}.
%As $s \rightarrow -\infty$, the power mean approaches the min function due to \eqref{eq:limit}, and thus, the objective agrees with the $k$-means criterion in this limit.

\citet{xu2019power} utilize these means toward clustering, proposing the power $k$-means objective function  defined
\vspace{-0.01cm}
\begin{equation}\label{eq:limit2}
f_s(\Theta )=\sum_{i=1}^n M_s(\|\bx_i-\btheta_1\|^2,\ldots,\|\bx_i-\btheta_k\|^2) \vspace{-0.1cm}
\end{equation}
for a given power $s$. The algorithm then seeks to minimize $f_s$ iteratively while sending $s \rightarrow -\infty$. Doing so approaches the original $f(\Theta)$ in \eqref{obj1} due to \eqref{eq:limit}, coinciding with the original $k$-means objective and retaining its interpretation as minimizing within-cluster variance. The $k$-harmonic means method \citep{zhang1999k}, an early attempt to reduce the sensitivity to initialization of $k$-means by replacing the $\min$ appearing in \eqref{obj1} by the harmonic average, % to yield a smoother optimization landscape. This method was noted to be effective only in low dimensions, and 
can be seen as the special case of \eqref{eq:limit2} with $s=-1$.
Power $k$-means clustering extends this idea to work in higher dimensions when the harmonic mean is longer a good proxy for \eqref{obj1}, instead using a sequence of \textit{power means} as a family of successively smoother optimization landscapes. 
The intermediate surfaces   exhibit fewer poor local optima than \eqref{obj1}, and each  step is carried out via MM.

\section{Bregman Power $k$-Means}
\label{problem statement}
We consider a power $k$-means objective function under a given Bregman divergence $d_\phi$ and power $s$:
\vspace{-0.1cm}
\begin{eqnarray}\label{obj}
    f_s(\bTheta) = \sum_{i=1}^n M_s(d_\phi(\bx_i, \btheta_1), \hdots, d_\phi(\bx_i, \btheta_k))
\end{eqnarray}

We see that this is a generalization of \eqref{eq:limit2}, as power $k$-means is recovered by taking $\phi$ to be the squared norm. On the other hand, \citet{paul2021uniform} propose to use Adagrad to minimize a more general objective
%\begin{equation}
%    \frac{1}{n}\sum_{i=1}^n M_s\left(d_\phi(\bX_i,\btheta_1),\dots,d_\phi(\bX_i,\btheta_k)\right) := f_{\bTheta} (\bX).
%\end{equation}
in that $M_s: \Real^k_{\ge 0} \to \Real_{\ge 0}$ is \textit{any} component-wise non-decreasing function, such as a generalized mean. Though the general theoretical treatment in \citet{paul2021uniform} does encompass the case where dissimiliarities $d()$ are given by Bregman divergences, the authors do not consider or implement this case explicitly. As a result, there a generic incremental optimization (such as Adagrad) is suggested, which produces a less scalable algorithm.  
In contrast, the geometry of Bregman divergences allows us to derive an elegant MM algorithm with closed form updates, matching the complexity of standard power $k$-means and Lloyd's algorithm. These properties will also lead to stronger theoretical results, detailed in the following section.

First, convexity of $\phi$ together with properties of power means ensures that our objective can be \textit{majorized} by its tangent plane. That is, upon differentiating \eqref{eq:firstpartial}, one can see that the Hessian matrix of $M_s()$ is concave whenever $s \leq 1$ \citep{xu2019power}.  This yields an upper bound which will supply a useful surrogate function:
\vspace{-0.1cm}
\begin{eqnarray}\label{eq:surrogate}
    \begin{split}
        f_s(\bTheta) \leq f_s(\bTheta_m) - \sum_{i=1}^n \sum_{j=1}^k w_{m,ij} \cdot d_\phi(\bx_i, \btheta_{m,j}) + \\
        \sum_{i=1}^n \sum_{j=1}^k w_{m,ij} \cdot d_\phi(\bx_i, \btheta_{j}),
    \end{split}
\end{eqnarray}
where the scalars from partial differentiation abbreviated \begin{eqnarray}\label{eq:w}
    w_{m,ij} = \frac{\frac1k d_\phi(\bx_i, \btheta_{m,j})^{s-1}}{(\frac1k \sum_{l=1}^k (d_\phi(\bx_l, \btheta_{m,j})^s)^{1 - \frac1s}}
\end{eqnarray} act as weights between $\bx_i$ and $\btheta_j$ at the $m^{th}$ iteration.
%\vspace{-0.1cm}
% \begin{eqnarray}
%     w_{m,ij} = \frac{\frac1k D_\phi(\bx_i, \btheta_{m,j})^{s-1}}{(\frac1k \sum_{i=1}^k (D_\phi(\bx_i, \btheta_{m,j})^s)^{1 - \frac1s}}
% \end{eqnarray}

Next, the mean-as-minimizer property from Prop. \eqref{prop:banerjee} suggests we may expect a closed form solution to the stationarity equations, which we derive here for completenesss.
Analogous to the iteration between updating cluster label assignments and then re-defining cluster means in Lloyd's algorithm for standard $k$-means, we update cluster centers by minimizing the right hand side of equation \eqref{eq:surrogate} given weights $w_{m,ij}$ with respect to $\boldsymbol\theta$: for each $j$, 
\begin{eqnarray}
    \begin{split}
        \nabla_{\btheta_j} \big[f_s(\bTheta_m) - \sum_{i=1}^n \sum_{j=1}^k w_{m,ij}  d_\phi(\bx_i, \btheta_{m,j}) + \\ 
        \sum_{i=1}^n \sum_{j=1}^k w_{m,ij}  d_\phi(\bx_i, \btheta_{j})\big] = 0
    \end{split} \\
    \sum_{i=1}^n w_{m,ij} \nabla^2_{\btheta_j} \phi (\btheta_j) \cdot [\btheta_j - \bx_i] = 0 \\
%    \theta_j \sum_{i=1}^n w_{m,ij} - \sum_{i=1}^n w_{m,ij} \bx_i = 0 \\ 
    \btheta_{m+1, j} = \frac{\sum_{i=1}^n w_{m,ij} \bx_i}{\sum_{i=1}^n w_{m,ij}}. \label{eq:theta}
\end{eqnarray}

\begin{algorithm}[h]
%\SetAlgoLined
1. Initialize $s_0 < 0$ and $\mathbf{\Theta_0}$, input data $\bx \in \mathbb{R}^{p \times n}$,  constant $\eta > 1$, iteration $m=1$ \\
2. \textbf{repeat} \\ 
3. $w_{m,ij} \leftarrow (\frac1k \sum_{i=1}^k d_\phi(\bx_i, \btheta_{m,j})^{s_m})^{\frac{1}{s_m} - 1}   d_\phi(\bx_i, \btheta_{m,j})^{s_m-1}$ \\
4. $\btheta_{m+1, j} = (\sum_{i=1}^n w_{m,ij})^{-1} \sum_{i=1}^n w_{m,ij} \bx_i$ \\
5. $s_{m+1} \leftarrow \eta \cdot s_m$ (optional) \\
6. \textbf{until} convergence
 \caption{Bregman Power $k$-means Pseudocode}
 \label{bpkm}
\end{algorithm}

Equations \eqref{eq:w} and \eqref{eq:theta} imply a transparent, easy-to-implement method that implicitly performs annealing through a family of optimization landscapes indexed by $s$. The resulting iteration can be summarized concisely in Algorithm 1. By contrast, a gradient-based update for $\btheta_j$ with step size $\alpha$ such as suggested in \citep{paul2021uniform} would entail
\begin{align*}
    \btheta_{m+1,j} = \btheta_{m,j} - \alpha\sum_{i=1}^n w_{m,ij} \nabla^2_{\btheta_{m,j}} \phi (\btheta_{m,j}) \cdot [\btheta_{m,j} - \bx_i],
\end{align*}
both incurring higher cost at each iteration and making significantly slower progress per step. Depending on the choice of method, computing and properly tuning $\alpha$ must be done on a case-by-case basis and adds additional overhead.

%Because we target the same Bregman hard clustering objective in the limit, the formulation inherits the guarantees of the global minimizers. On the other hand, in practice by annealing through MM we are able to better avoid poor minimizers.

%\bgroup
%\linespread{1.0}%  1 is the default, change whatever you need
% \begin{table}[!t]
%     \caption{Some Exponential Family Distributions and Their Corresponding Bregman Divergences}\label{tab:some}
% 	\centering
% 	\begin{tabular}{|p{1.0cm}|p{1.75cm}|p{2.5cm}|p{2.7cm}| }
%          \hline
%          Domain & Distribution & $\phi(\bx)$ & $d_\phi(\bx,\by)$ \\
%          \hline
%          $\mathbb{R}^m$ & Gaussian & $\|\bx\|^2$ &  $\|\bx-\by\|^2$\\
%          $\mathbb{R}^m$ & Multinomial & $\sum_{i=1}^m x_i \log x_i$ & $\sum_{i=1}^m x_i \log \frac{x_i}{y_i}$\\
%          $\mathbb{R}^1$ & Gamma & $-k + k \log \frac{k}{\bx}$ & $\frac{k}{\by} (\by \log \frac{\by}{\bx}) + \bx - \by$\\
%          $\mathbb{R}^1$ & Poisson & $x \log x - \bx$ & $x \log \frac{x}{y} - (x-y)$\\
%          \hline
%     \end{tabular}
% \end{table}

\begin{table}[!t]
    \caption{Examples of exponential family distributons and their corresponding Bregman divergences}\label{tab:some}
	\centering
	\begin{tabular}{|p{1.8cm}|p{2.25cm}|p{2.9cm}| }
         \hline
           Distribution & $\phi(\bx)$ & $d_\phi(\bx,\by)$ \\
         \hline
          Gaussian & $\|\bx\|^2$ &  $\|\bx-\by\|^2$\\
           Multinomial & $\sum_{i=1}^m x_i \log x_i$ & $\sum_{i=1}^m x_i \log \frac{x_i}{y_i}$\\
           Gamma & $-\alpha + \alpha \log \frac{\alpha}{x}$ & $\frac{\alpha}{y} (y \log \frac{y}{x}) + x - y$\\
           Poisson & $x \log x - x$ & $x \log \frac{x}{y} - (x-y)$\\
         \hline
    \end{tabular}
\end{table}
% \textbf{Note: for the Gamma distribution, $k$ corresponds to the Gamma shape parameter}

%\egroup
\paragraph{Exponential family data} 
A statistical motivation for our generalization comes from the connection between Bregman divergences and exponential families. % suggesting this approach is well-suited for such data generating mechanisms. 
Recall exponential family distributions with parameter $\theta$ and scale parameter $\tau$ take the canonical form
$$p( y | \theta, \tau ) = C_1(y,\tau) \exp \left\{ \frac{ y \theta - \phi^\ast( \theta) }{C_2(\tau)} \right\}. $$ 
The convex conjugate of its \textit{cumulant function} $\phi^\ast$, which we denote $\phi$, uniquely generates the Bregman divergence $d_\phi$ that represents the exponential family likelihood up to proportionality. With $g$ denoting the canonical link function, %\cite{xu2018majorization}, 
the negative log-likelihood of $y$ can be written as its Bregman divergence to the mean:
\[  - \ln p( y | \theta, \tau) = d_\phi \left( y , g^{-1} (\theta ) \right) + C(y, \tau). \]
As an example, the cumulant function in the Poisson likelihood is $\phi^\ast(x) = e^x$, whose conjugate $\phi(x)=x \ln x - x$ produces the relative entropy $d_\phi(p,q) = p \ln (p/q) - p + q.$  Similarly, recall that the Bernoulli likelihood has cumulant function $\phi^\ast(x) = \ln (1 + \exp(x))$. Its conjugate is given by $\phi(x) = x \ln  x + (1-x) \ln (1-x)$, and generates $d_\phi(p,q) = p \ln\frac{p}{q} + (1-p) \ln\frac{1-p}{1-q}.$ 
These relationships for some common distributions are summarized in Table \ref{tab:some} and show, for instance, that maximizing the likelihood of a generalized linear model is equivalent to minimizing a Bregman divergence between the responses and regression parameters. In the context of clustering, they allow us to understand the analog of $k$-means minimizing the within-cluster variance. Indeed, the Bregman hard clustering problem is equivalent to finding a partitioning of the data such that the loss in Bregman information $I_\phi(X)$ due to quantization is minimized---or equivalently, such that the within-cluster Bregman information is minimized. See Theorem 1 of \citet{banerjee2005clustering} for details and a formal statement of this result.
Because we target the same objective \eqref{obj1} as $s \rightarrow -\infty$, our formulation inherits this property immediately in the target limit. In fact, it is not difficult to show that this convergence is \textit{uniform}:
\begin{theorem}\label{o2}
For any sequence, $s_m \downarrow -\infty$ and $s_1 \le 1$, $f_{s_m}(\cdot)$ converges uniformly on $\C$ to the Bregman hard clustering objective \eqref{obj1}. %f_{-\infty}(\cdot)$, uniformly on $\C$. 
\end{theorem}

Another desirable property of Algorithm \ref{bpkm} is that all iterates lie within the convex hull of the data, which suggests performance stability in addition to standard convegence and descent guarantees as a valid MM algorithm \cite{lange2016mm}. Let $\mathscr{C}$ denote the closed convex hull of the data; the following result is inherited directly from  power $k$-means  \citep{xu2019power} %to Bregman divergence as follows. 
%Theorems \ref{o1} and \ref{o2} imply that the sequence of functions $f_s(\cdot)$, uniformly approximates the objective $f_{-\infty}(\cdot)$ on the closed convex hull of the data, which contains the (global) minimizer.  \ref{optimization proofs}.
\begin{theorem}\label{o1}
Let $\bTheta_{n,s}$ be the (global) minimizer of $f_s(\cdot)$. Then $\bTheta_{n,s} \subset \mathscr{C}$ for all $s \le 1$.
\end{theorem}
These proofs are fairly straightforward and are given in full detail in the Appendix.

\section{Theoretical Analysis}
\label{stat theory}
In addition to casting the problem in such a way that it inherits classical guarantees, we now contribute new theoretical devices toward understanding its generalization error. 
%We now focus on the finite sample properties of the obtained estimator of the cluster centroids. 
The complete proofs pertaining to this section are available in the Appendix. We consider data  $\{\bX_i\}_{i \in [n]}$ independent and identically distributed according to some distribution $P$, and further assume that $P$ has a sub-exponential $\ell_2$ norm. This condition on $P$ is strictly weaker than imposing that  $P$ has bounded support, as required in recent analyses in the literature \cite{paul2021uniform}: formally,

\begin{assumption}
\label{ass1}
$\{\bX_i\}_{i \in [n]} \overset{\text{i.i.d.}}{\sim} P$, with 
\begin{itemize}
    \item $\sigma = \|\|\bX\|_2\|_{\psi_1} \triangleq \sup_{p \in \mathbb{N}}\frac{(\E \|\bX\|_2^p)^{1/p}}{p} < \infty.$
    \item $\sigma_\phi = \|\phi(\bX)\|_{\psi_1} \triangleq \sup_{p \in \mathbb{N}}\frac{(\E |\phi(\bX)|^p)^{1/p}}{p} < \infty$.
\end{itemize}
%\[\sigma = \|\|\bX\|_2\|_{\psi_1} := \sup_{p \in \mathbb{N}}\frac{(\E \|\bX\|_2^p)^{1/p}}{\sqrt{p}} < \infty.\]
\end{assumption}
Note that A~\ref{ass1} is satisfied by many popularly used distributional models not limited to Gaussian mixtures, and always holds whenever $P$ has bounded support. We also make the following standard assumption on regularity of the corresponding Bregman divergence. % we make the following standard assumption. 
\begin{assumption}
\label{ass2}
$\nabla \phi$ is $\tau_2$-Lipschitz. Moreover, $\phi$ is $\tau_1$-strongly convex, i.e. $\forall \, \bx,\by \in \Real^p$ and for $0 \le \alpha \le 1$,
{\small
\[\phi(\alpha \bx +(1-\alpha) \by) \le \alpha \phi(\bx) + (1-\alpha) \phi(\by) - \frac{\tau_1}{2}\alpha (1-\alpha) \|\bx - \by\|_2^2.\]
}%
\end{assumption}
Recall strong convexity of $\phi$ relates to the smoothness of its conjugate, i.e. the cumulant function $\phi^\ast$ of exponential families \citep{kakade2010learning,zhou2018fenchel}. 
%These assumption is quite natural and hold for most commonly used Bregman divergences, trivially holding for squared Euclidean distances and on a compact set $[\delta,M]$, with $M > \delta>0$ for all other examples in Table \ref{tab:some}, when the data is assumed to be generated within that compact support. 
Note that under A~\ref{ass2}, both $\sigma$ and $\sigma_\phi$ are finite when $\|\bX\|_2$ is sub-Gaussian. Thus, we are also able to generalize the assumptions used in analyses of approaches such as convex clustering \cite{tan2015statistical}, as detailed in Appendix \ref{subg}.  %One immediate consequence of Assumptions \ref{ass1} and \ref{ass2} is that $\phi(\bX)$ is $\tau_2 \sigma$-sub-exponential.
% \begin{lemma}\label{3.1}
% Under A\ref{ass1} and A\ref{ass2}, $\sigma_\phi \triangleq \|\phi(\bX)\|_{\psi_1}~\le~\tau_2 \sigma$.
% \end{lemma}
Now, let $\tilde{f}_{\bTheta}(\bx) = M_s\left(d_\phi(\bx,\btheta_1), \dots, d_\phi(\bx,\btheta_k)\right)$ and $P_n$ be the empirical distribution based on the data $\{\bX_i\}_{i \in [n]}$. That is, $P_n(A) = \frac{1}{n}\sum_{i=1}^n \one\{\bX_i \in A\}$ for any Borel set $A$. For simplicity, we denote $\mu g = \int g d\mu$ for any measurable function $g$ and measure $\mu$. Fixing these conventions, note the objective \eqref{obj}, upon scaling by $1/n$, can be written as $P_n \tilde{f}_{\bTheta}$. By the strong law, we know that for any $\bTheta \in \Real^{k \times p}$, $P_n \tilde{f}_{\bTheta} \xrightarrow{a.s.} P \tilde{f}_{\bTheta}$. Thus, as $n$ becomes large, our intuition tells us to expect that as the functions of $\bTheta$ $P_n \tilde{f}_{\bTheta}$ and $P \tilde{f}_{\bTheta}$, become close to each other, so do their respective minimizers $\hth$ and $\tth$. To make precise the notion of convergence of $\hth$ towards $\tth$, we denote the \textit{excess risk} at any set of cluster centroids $\bTheta$ as
\[\mathfrak{R}(\bTheta) = P \tilde{f}_{\bTheta} - P \tilde{f}_{\tth}.\] 
The goal of this section is to formally assert that $\mathfrak{R}(\hth)$ becomes very small with a high probability as one has access to more and more data. 

As a first step, we prove a high probability result on $\hth$, showing that $\hth$ remains bounded with a high probability as $n$ becomes large. To this end, we require a notion of distance between sets of cluster centroids, and following the literature \citep{chakraborty2021uniform} use the measure
%establish this result we need to impose a regularity condition on $P$ and also require a notion of distance between two cluster centroids. We use the following measure of dissimilarity popular in literature \cite{chakraborty2021uniform,paul2021uniform,paul2021on}.
\[\text{dist}(\bTheta_1,\bTheta_2) \triangleq \min_{O \in \mathscr{P}_k} \|\bTheta_1 - O \bTheta_2\|_F,\]
where $\mathscr{P}_k$ denotes the set of all $k \times k$ real permutation matrices. In particular, this accounts for the label switching problem and is agnostic to relabeling classes. Likewise, we require a standard identifiability condition \cite{pollard1981strong,paul2021uniform}:
\begin{assumption}\label{ass3}
Let 
\(M_\epsilon = \inf\{M>0 : \{\bTheta \in \Real^{k \times p}:\text{dist}(\bTheta,\tth)>M\} \subseteq  \{\bTheta \in \Real^{k \times p}: \mathfrak{R}(\bTheta) > \epsilon\}\}\). Then for any $\epsilon>0$, we have  $M_\epsilon < \infty$. %there exists $M_\epsilon>0$, such that $\text{dist}(\bTheta,\tth)>M_\epsilon, \implies \mathfrak{R}(\bTheta) > \epsilon$.
\end{assumption}
This states that when the distance from $\bTheta$ to $\tth$ is large, then the excess risk at $\bTheta$ is also large. The following theorem formally states a high probability bound on $\hth$.
\begin{theorem}\label{t1}
Under assumption \ref{ass1}-\ref{ass3}, $\hth \subset B(\xi_P + \|\tth\|_F)$, with probability at least $1- e^{-cn}$. Here $c$ is an absolute constant and $\xi_P = M_{P \phi + \sigma_\phi}$.
\end{theorem}
The main idea for the proof of Theorem \ref{t1} is that $P \tilde{f}_{\hth}$ remains bounded with a high probability. Thus, $\mathfrak{R}(\hth)$ is also bounded with a high probability, which in turn implies that $\text{dist}(\hth,\tth) \le \xi_P$.
Before our main theorem, we recall the definitions of Rademacher complexity \cite{bartlett2002rademacher} and covering numbers.
\begin{definition}(Rademacher complexity)
The population Rademacher complexity of a function class $\mathcal{F}$ is defined as,
\[\mathcal{R}_n(\F) = \frac{1}{n}\E \sup_{f \in \F} \sum_{i=1}^n \epsilon_i f(\bX_i), \]
where $\epsilon_i$'s are i.i.d Rademacher random variables.
\end{definition}
\begin{definition}($\delta$-cover and covering number) For a metric space $(X,d)$, the set $X_\delta \subseteq X$ is said to be a $\delta$-cover of $X$ if for all $x \in X$, there is $x^\prime \in X_\delta$, such that $d(x,x^\prime) \le \delta$. The $\delta$-covering number of $X$, denoted by $N(\delta;X,d)$, is the size of the smallest $\delta$-cover of $X$ with respect to  $d$. 
\end{definition}

Now consider the set $\F = \{\tilde{f}_{\btheta}: \bTheta \subset B(\xi_P + \|\tth\|_F)\}$, under the  measure of distances between functions
\[d_{2n}(f,g) \triangleq \left(\frac{1}{n} \sum_{i=1}^n (f(\bX_i) - g(\bX_i))^2\right)^{-1/2}.\]

The following theorem imposes a bound on the covering number of $\F$ with respect to the $d_{2n}$ metric. The proof makes use of A~\ref{ass2}, that $\tilde{f}_{\bTheta}$ is Lipschitz  on $B(\xi_P + \|\tth\|_F)$.
\begin{theorem}\label{entropy} Under assumption A~\ref{ass2},
\[\mathcal{N}(\delta;\mathcal{F}, d_{2n} ) \le \left(\max\left\{1,\left\lfloor \frac{(\xi_P + \|\tth\|_F) C^{1/2}}{\delta}\right\rfloor\right\}\right)^{kp}; \]
$C = 2 k^{2-2/s} \tau_2^2   p  n^{-1}\sum_{i = 1}^n ( 18 \xi_P^2  + 18  \|\tth\|_F^2 +  \|\bX_i\|_2^2 )$.
\end{theorem}
One can now appeal to this bound on the covering number of $\F$ to bound the Rademacher complexity, $\mathcal{R}_n(\F)$. More technically, we make use of Theorem \ref{entropy} and apply  Dudley's chaining arguments to produce a $\mathcal{O}(1/\sqrt{n})$ bound on the Rademacher complexity of $\F$.
\begin{theorem}
\label{rad}
Under assumptions A~\ref{ass1} and A~\ref{ass2}, 
\begin{align*}
    \mathcal{R}_n(\mathcal{F}) \, \,\le & \, \, 6 \tau_2 C^\prime(\xi_P + \|\tth\|_F) \frac{k^{3/2 - 1/s}p}{\sqrt{n}},
\end{align*}
where $C^\prime = \sqrt{2 \pi  ( 18 \xi_P^2  + 18  \|\tth\|_F^2 + \E\|\bX\|_2^2 )}$.
\end{theorem}
The Rademacher complexity bound plays a key role in providing uniform concentration bounds on $\|P_n - P\|_{\F} = \sup_{f \in \F} |P_n f - P f|$. Since the functions in $\F$ are not bounded (as we do not assume $\bx$ is bounded), the classical results by \citet{bartlett2002rademacher} do not  directly apply. However, appealing to the sub-exponential property of $\|\bX\|_2$, we  apply recent concentration results derived by \cite{maurer2021concentration}. Formally, our bound is as follows:

\begin{theorem}\label{concentration}
Suppose assumptions A~\ref{ass1}-\ref{ass2} hold. Then for $n \ge \log(2/\delta) \ge \frac{1}{2}$, with probability at least $1-\delta$,
\begin{align*}
   & \|P_n  - P \|_{\F}  \, \,\le\, \, 12 \tau_2  C^\prime (\xi_P + \|\tth\|_F) \frac{k^{3/2 - 1/s}p}{\sqrt{n}} \\
 & +16 e \sigma \tau_2 k^{1-1/s}(1 + \xi_P + \|\tth\|_F  ) \sqrt{\frac{2 \log (2/\delta)}{n}}.
\end{align*}
\end{theorem}
From Theorem \ref{t1}, we know that with a very high probability, $\hth \subset B(\xi_p + \|\tth\|_F)$. Using this result, it is not difficult to then show that with a very high probability, $\mathfrak{R}(\hth) \le 2 \|P_n - P\|_{\F}$, which can be bounded by Theorem \ref{concentration}. Finally, the next theorem provides a bound on the excess risk.
\begin{theorem}\label{er}
Let Assumptions \ref{ass1}-\ref{ass2} hold. Then whenever $n \ge \log(2/\delta) \ge \frac{1}{2}$, with probability at least $1-\delta - e^{-cn}$,
\begin{align*}
   & \mathfrak{R}(\hth)  \, \, \le \, \, 24 \tau_2  C^\prime (\xi_P + \|\tth\|_F) \frac{k^{3/2 - 1/s}p}{\sqrt{n}} \\
 & +32 e \sigma \tau_2 k^{1-1/s}(1 + \xi_P + \|\tth\|_F  ) \sqrt{\frac{2 \log (2/\delta)}{n}}.
\end{align*}
\end{theorem}
\paragraph{Remark} It is important to note that the bounds derived in Theorem \ref{er} include the Frobenius norm of the population cluster centroid $\|\tth\|_F$, as well as terms such as $\E \|\bX\|_2^2$ and $\|\|\bX\|_2\|_{\psi_1}$ measuring the spread of the data. Intuitively, as spread of the data increases, it can be expected that the performance of Bregman power $k$-means  deteriorates with the added noise. This phenomenon is reflected in the bounds on the excess risk.

\paragraph{Strong Consistency and $\sqrt{n}$-consistency}
In the classical domain of keeping $k$ and $p$ fixed, one can recover asymptotic results \cite{pollard1981strong} such as strong consistency and $\sqrt{n}$-consistency of the sample cluster centroids.  We say that the sequence of the set of cluster centroids $\{\bTheta_n \}_{n \in \mathbb{N}}$ converges to $\bTheta$ if $\lim_{n \to \infty}\text{dist}(\bTheta_n, \bTheta ) =0$. %This notion of convergence was previously used by \cite{paul2021on,paul2021uniform,chakraborty2021uniform} 
The following theorem asserts that indeed $\hth$ is strongly consistent for $\bTheta$, and moreover admits a parametric convergence rate of $\mathcal{O}(n^{-1/2})$. Before stating the result, recall that for a sequence of random variables $\{X_n\}_{n \in \mathbb{N}}$, we say that $X_n = \mathcal{O}_P(a_n)$, for a sequence of reals $\{a_n\}_{n \in \mathbb{N}}$, if $X_n/a_n$ is \textit{tight}, or bounded in probability.

\begin{theorem}\label{consistency}
If $p$ is kept fixed, then under Assumptions \ref{ass1}-\ref{ass3}, $\hth \xrightarrow{a.s.} \tth$. Moreover, $\mathfrak{R}(\hth) = \mathcal{O}_P(n^{-1/2})$.
\end{theorem}
\section{Empirical Performance and Results}
We close our assessment of the proposed method with a thorough empirical study. Previous works on Bregman clustering largely focus on  mathematical aspects, and the few that include data examples are limited to low dimensions. We extend their designs  to settings as dimension increases, for a large number of clusters, and for a breadth of exponential family distributions with varying parameters, followed by application to a rainfall dataset. 
An open-source Python implementation of the proposed method, including reproducible code for all data generating mechanisms and experiments in this paper, is available and maintained in a repository by the first author\footnote{Publicly available at \url{https://github.com/avellal14/bregman_power_kmeans}}.
 \begin{figure}[h] \vspace{-5pt}
   \hspace{-20pt}
    \includegraphics[width=9.9cm]{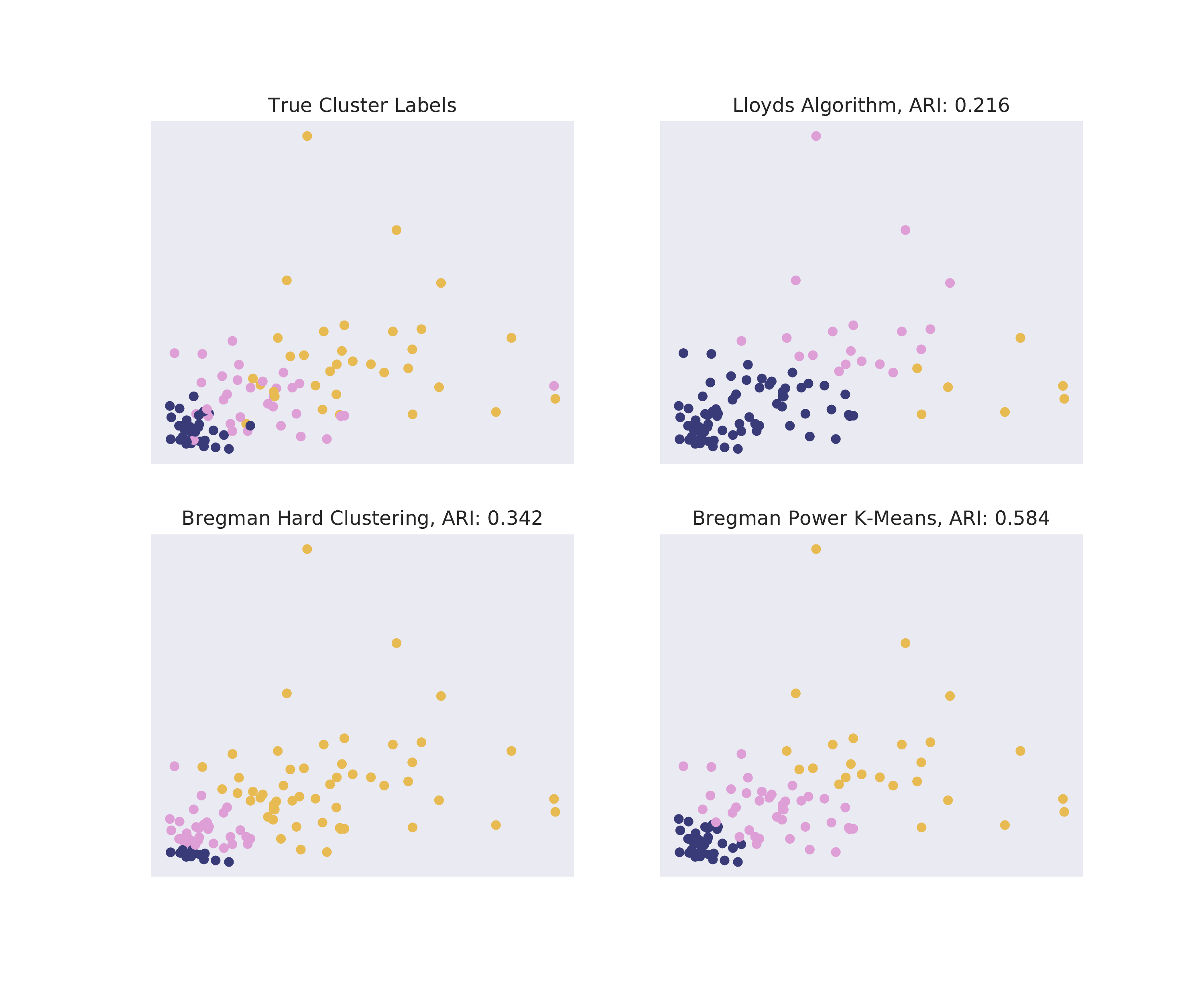}
    \vspace{-30pt}\caption{A visual comparison of clustering solutions.} \label{fig:visual}
\end{figure}

\textbf{Experiment 1:} We begin by considering data simulated from various exponential families in the plane. Synthetic datasets are generated in $\mathbb{R}^2$ from Gaussian, Binomial, Poisson, and Gamma distributions from true centers at $(10,10)$, $(20,20)$, and $(40,40)$. % 33 data points were sampled from each of 3 clusters whose centers were located at (10,10), (20,20), and (40,40) respectively. 
For the normal case $\sigma^2=16$, while the binomial parameter $n=200$ so that clusters feature heteroskedasticity with variance implied by the mean relationship. Poisson data are sampled coordinate-wise  coordinate with intensity parameter $10,20$, and $40$ respectively. Finally, each Gamma coordinate was sampled to have the same means, with shape parameters fixed at $\alpha=15$. % distribution with shape parameter $k=15$ centered at 10,20, or 40 respectively. 

%For the Gaussian data, a cluster variance of 16 was used. For the Binomial data, an $n$ parameter of 200 was used such that the cluster at (10,10) had a variance of 9.5, the cluster at (20,20) had a variance of 18.0, and the cluster at (40,40) had a variance of 32.0. Poisson data was constructed by sampling each coordinate from a one-dimensional Poisson distribution centered at 10,20, or 40 respectively. Finally, each Gamma coordinate was sampled using a one-dimensional Gamma distribution with shape parameter $k=15$ centered at 10,20, or 40 respectively. 

As discussed in Section 2 and illustrated in Table 1, a Bregman divergence for each of these distributions may provide an ideal measure of dissimilarity for clustering data generated under its corresponding exponential family. To investigate this, we apply Lloyd's $k$-means algorithm, Bregman hard clustering, the original power $k$-means method, and our proposed Bregman power $k$-means algorithm on each of these four settings. Centers are randomly initialized according to a uniform distribution spanning the range of all the data points, and each peer method starts from matched initializations to ensure a fair comparison. An $s_0$ value of $-0.2$ was used for power $k$-means and our method. 

A visual comparison of clustering solutions obtained on one simulated dataset by each peer method on Gamma data with a shape parameter $\alpha=5$ is displayed in Figure \ref{fig:visual}. Comparing to the ground truth labels, this illustrates that Bregman power $k$-means is much more effective than its competing methods in distinguishing between points drawn from different Gamma clusters, despite the data not being perfectly separable. This is particularly apparent for Lloyd's algorithm, which fails due to seeking spherical clusters. For a closer look, we report the mean (and standard deviation) adjusted Rand index (ARI) of solutions under each algorithm computed over 250 trials in Table \ref{tab:ARI}. We observe that other than performing on par in the Gaussian case, Bregman power $k$-means consistently achieves the best performance in the other exponential family settings.

% \textbf{Experiment 2:} To evaluate and compare the performance of Lloyd's Algorithm, Power $k$-means and Bregman Power $k$-means on slightly harder clustering problems, the data generation setup described in Experiment 1 was altered such that 11 data points were sampled from each of 9 clusters. The proximity of these cluster centers was altered over 3 "levels". Level 1 corresponded to the 9 cluster centers being placed uniformly in the box $[10,20]^2$. Level 2 corresponded to the 9 cluster centers being placed uniformly in the box $[10,30]^2$. For level 3, the 9 cluster centers were placed uniformly in the box $[10,50]^2$. An $s_0$ value of -0.2 was used for power $k$-means and Bregman power $k$-means, and ARIs were computed over 250 random trials.

% \begin{figure}[h!]
%     \centering
%     \caption{Experiment 2: Clustering Performance With Varying Cluster Proximity}
%     \includegraphics[width=8cm]{ARI_vs_cluster_proximity.png}
% \end{figure}

%TODO for Adithya: clean up description + explain significance

\begin{table}[h]
	\centering
	\caption{Mean and (standard deviation) ARI of Lloyd's algorithm, Bregman hard clustering, and their power means counterparts.}
	
	\begin{tabular}{|p{1.2cm}|p{1.0cm}|p{1.2cm}|p{1.0cm}|p{1.2cm}|}
         \hline
          & Lloyd's & Bregman Hard  & Power  & Bregman Power  \\
         \hline
         Gaussian & 0.828 (0.012) & 0.837 (0.012) & 0.927 (0.003) & 0.927 (0.003) \\
         Binomial & 0.730 (0.014) & 0.886 (0.011) & 0.915 (0.004) & 0.931 (0.003) \\
         Poisson & 0.723 (0.014) & 0.882 (0.010) & 0.888 (0.006) & 0.916 (0.004) \\
         Gamma & 0.484 (0.009) & 0.868 (0.005) & 0.677 (0.008) & 0.879 (0.004)\\
         \hline
    \end{tabular} 	\label{tab:ARI}

\end{table}

\textbf{Experiment 2:} To better understand the behavior of each clustering method as the shape of distributions change, we now revisit the Gamma setting while varying the shape parameters  $\alpha=1,\ldots,20$. We also increase the problem dimension to $p=20$, with all other simulation details are unchanged from Experiment 1. Since the centers are held fixed as $\alpha$ is increased, higher $\alpha$ values correspond to less skewed Gamma distributions with lower variances. The mean ARIs (each again computed over 250 random simulations) against increasing shape parameter values are summarized in Figure \ref{fig:gamma}. Due to the high skewness causing significant overlap between clusters, smaller shape parameters result in the poorest performance across all four peer methods. Each method seems to reach an inflection point somewhere in the range of $\alpha=3-6$, after which further increases in $\alpha$ minimally change the overall shape of the distribution. As expected, Bregman power $k$-means achieves the best performance while Lloyd's algorithm struggles across the board. %Bregman hard clustering is the only of the four methods which exhibits any downward trend in performance as $k$ increases.  Mean ARIs decrease by about $10\%$ as $k$ moves from 3 to 20. 
It is worth also noting that  (Euclidean) power $k$-means eventually overtakes Bregman hard clustering, and mostly maintains better performance than Bregman hard clustering for high shape parameter values. This at first surprising result can be reconciled by the interpretation of Gamma distributions as a sum of exponential random variables (explicitly when $\alpha$ is an integer), so that a sum of i.i.d variables looks closer to normal as $\alpha$ increases. Eventually, the improvement bestowed by annealing through poor minima becomes more advantageous than information that the data are in fact Gamma as they resemble Gaussian data more and more closely. The gap in performance between our proposed method and Bregman hard clustering reiterates the merits of annealing through power means. 

%This suggests that for such high shape parameter values, the geometry of Gamma distributions are similar enough to Gaussian distributions such that the improved optimization properties provided by Power $k$-means outweigh the adaptiveness to Gamma distribution geometry offered by Bregman hard clustering.

\begin{figure}[h]
    \centering
    \includegraphics[width=8cm]{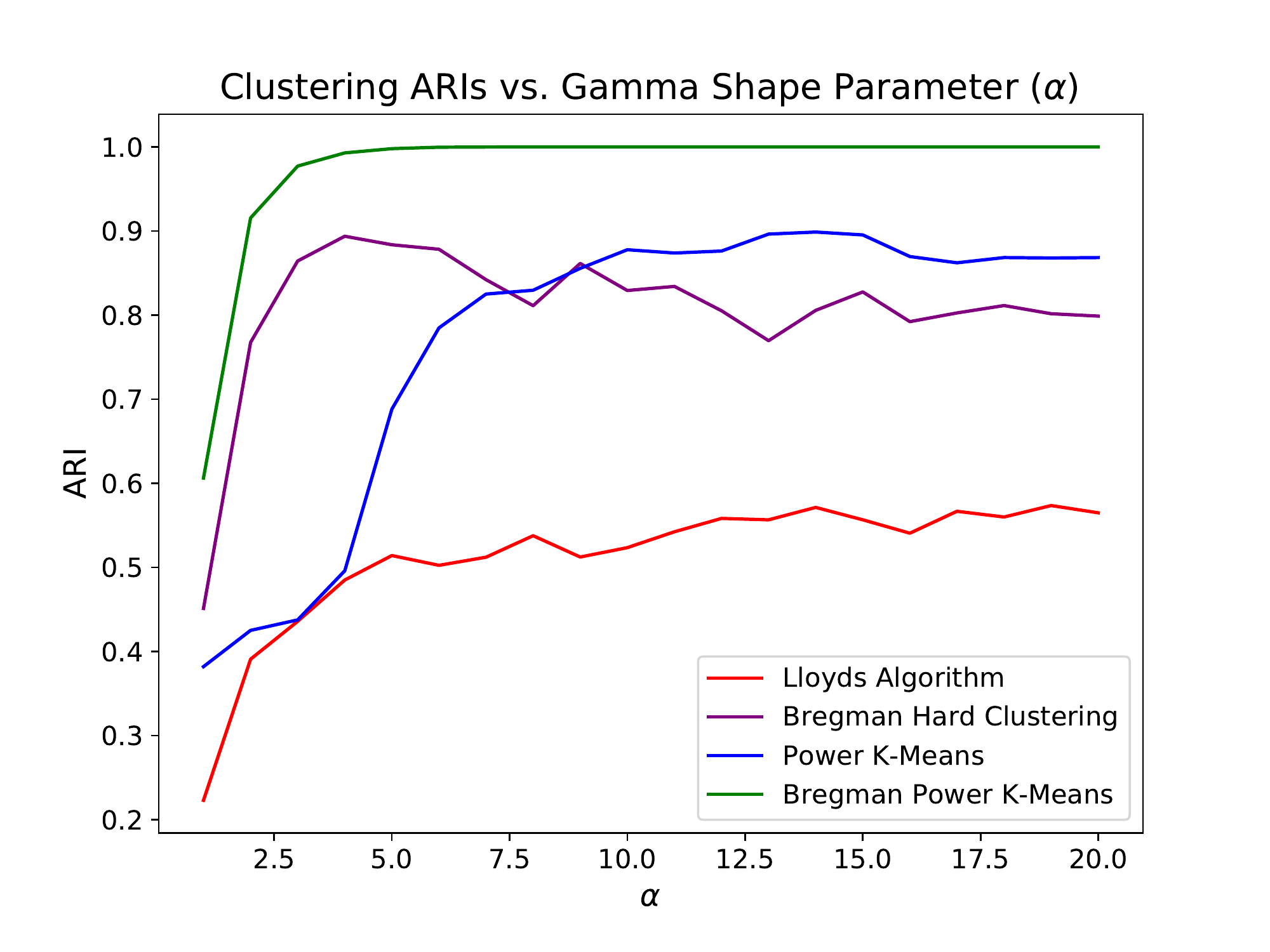}
    \caption{Performance as Gamma shape parameter varies.} 
    \label{fig:gamma} \vspace{-1pt}
\end{figure}

%\bgroup
%\linespread{1.0}%  1 is the default, change whatever you need
\begin{table*}[t]
	\centering
	\caption{Average (standard deviation) ARI across 250 trials as dimension increases, Poisson data.}
	\begin{tabular}{p{3.35cm}|p{1.95cm} p{1.95cm} p{1.85cm} p{1.95cm} p{1.95cm} p{1.95cm}}
        % \hline
          & $p=2$ & $p=5$ & $p=10$ & $p=20$ & $p=50$\\
         \hline
         Lloyd's & 0.418 (0.005) & 0.488 (0.006) & 0.610 (0.009) & 0.607 (0.012) & 0.614 (0.015)\\
         \hline
         Bregman Hard & 0.425 (0.005) & 0.497 (0.006) & 0.613 (0.009) & 0.653 (0.011) & 0.720 (0.014)\\
         \hline
         Power, \ $s_0=-0.2$ & 0.451 (0.005) & \bf{0.577} (0.005) & 0.690 (0.005) & 0.561 (0.010) & 0.440 (0.001)\\
         \hline
         Breg. Power, \ $s_0=-.2$ & \bf{0.458} (0.005) & \bf{0.577} (0.005) & 0.718 (0.005) & 0.734 (0.010) & 0.442 (0.002)\\
         \hline
         Power, \ $s_0=-1$ & 0.445 (0.005) & 0.575 (0.005) & 0.692 (0.005) & 0.745 (0.005) & 0.485 (0.008)\\
         \hline
         Breg. Power, \ $s_0=-1$ & 0.452 (0.005) & 0.575 (0.005) & \bf{0.720} (0.005) & \bf{0.806} (0.005) & 0.610 (0.014)\\
         \hline
         Power, \ $s_0=-3$ & 0.432 (0.005) & 0.548 (0.006) & 0.696 (0.005) & 0.784 (0.004) & 0.873 (0.004) \\
         \hline
         Breg. Power, \ $s_0=-3$ & 0.412 (0.005) & 0.555 (0.005) & 0.712 (0.005) & 0.804 (0.005) & 0.924 (0.003)\\
         \hline
         Power, \ $s_0=-9$ & 0.412 (0.005) & 0.531 (0.006) & 0.693 (0.006) & 0.782 (0.006) & 0.915 (0.003)\\
         \hline
         Breg. Power, \ $s_0=-9$ & 0.321 (0.006) & 0.510 (0.007) & 0.696 (0.006) & 0.792 (0.005) & \bf{0.925} (0.003)\\
         \hline
    \end{tabular} \vspace{-10pt} \label{tab:dim}
\end{table*}

\begin{table*}[t]
	\centering
	\caption{Average (standard deviation) runtimes (sec) across 250 trials as dimension increases, Poisson data.}
	\begin{tabular}{p{3.35cm}|p{1.95cm} p{1.95cm} p{1.85cm} p{1.95cm} p{1.95cm} p{1.95cm}}
        % \hline
          & $p=2$ & $p=5$ & $p=10$ & $p=20$ & $p=50$\\
         \hline
         Lloyd's & 0.082 (0.004) & 0.067 (0.002) & 0.058 (0.002) & 0.052 (0.001) & 0.056 (0.002)\\
         \hline
         Bregman Hard & 0.607 (0.012) & 0.493 (0.008) & 0.413 (9e-4) & 0.439 (0.005) & 0.506 (0.008)\\
         \hline
         Power, \ $s_0=-0.2$ & 0.012 (0.001) & 0.001 (3e-4) & 0.009 (2e-4) & 0.007 (4e-4) & 0.004 (1e-4)\\
         \hline
         Breg. Power, \ $s_0=-.2$ & 0.022 (0.001) & 0.020 (8e-4) & 0.017 (4e-4) & 0.021 (7e-4) & 0.009 (2e-4)\\
         \hline
         Power, \ $s_0=-1$ & 0.006 (1e-4) & 0.005 (1e-4) & 0.007 (1e-4) & 0.016 (4e-4) & 0.006 (3e-4)\\
         \hline
         Breg. Power, \ $s_0=-1$ & 0.011 (3e-4) & 0.012 (3e-4) & 0.014 (3e-4) & 0.036 (7e-4) & 0.019 (9e-4)\\
         \hline
         Power, \ $s_0=-3$ & 0.005 (1e-4) & 0.006 (1e-4) & 0.007 (2e-4) & 0.0059 (1e-4) & 0.008 (2e-4) \\
         \hline
         Breg. Power, \ $s_0=-3$ & 0.007 (1e-4) & 0.01 (2e-4) & 0.013 (4e-4) & 0.011 (3e-4) & 0.017 (4e-4)\\
         \hline
         Power, \ $s_0=-9$ & 0.004 (9e-5) & 0.006 (2e-4) & 0.011 (3e-4) & 0.055 (1e-4) & 0.007 (2e-4)\\
         \hline
         Breg. Power, \ $s_0=-9$ & 0.004 (2e-5) & 0.01 (2e-4) & 0.019 (4e-4) & 0.010 (2e-4) & 0.012 (3e-4)\\
         \hline
    \end{tabular} \vspace{-10pt} \label{tab:dim_times}
\end{table*}

\textbf{Experiment 3:} To better understand the effect of annealing in various feature dimensions $p$, we take a closer look at the Poisson setting when dimensionality ranges from $p=2$ to $p=50$. In this setting, true centers are given by $[40,40]$, $[50,50]$ and $[60,60]$ in the planar case, while we scale the separation inversely by a $\sqrt{d}$ factor as is standard to avoid the problem becoming ``too easy" in larger dimensions \cite{aggarwal2001surprising}. Mean ARIs and standard deviations across 250 random trials are detailed in Table \ref{tab:dim}, which also considers various initial powers $s_0$ for power $k$-means and our proposed method. We reproduce the finding of \citet{xu2019power} under the Gaussian setting that while annealing provides advantages without the need to carefully tune $s_0$, more negative starting values tend to yield a more pronounced advantage as $p$ increases.  The best performing method for each dimension $p$ is boldfaced. Similarly, we provide runtime details in Table \ref{tab:dim_times}. Though results vary based on implementations, we see that the proposed method is highly efficient in the data settings we consider here, and outpace competitors by an order of magnitude. Similar trends are conveyed under the other data generating mechanisms. %The pattern of lower $s_0$ values offering better performance as $d$ increases is consistent with what was observed for the Gaussian case in \cite{xu2019power}.

%TODO for Adithya: clean up description + explain significance
\begin{figure}[h] \vspace{-8pt}
    \includegraphics[width=8cm]{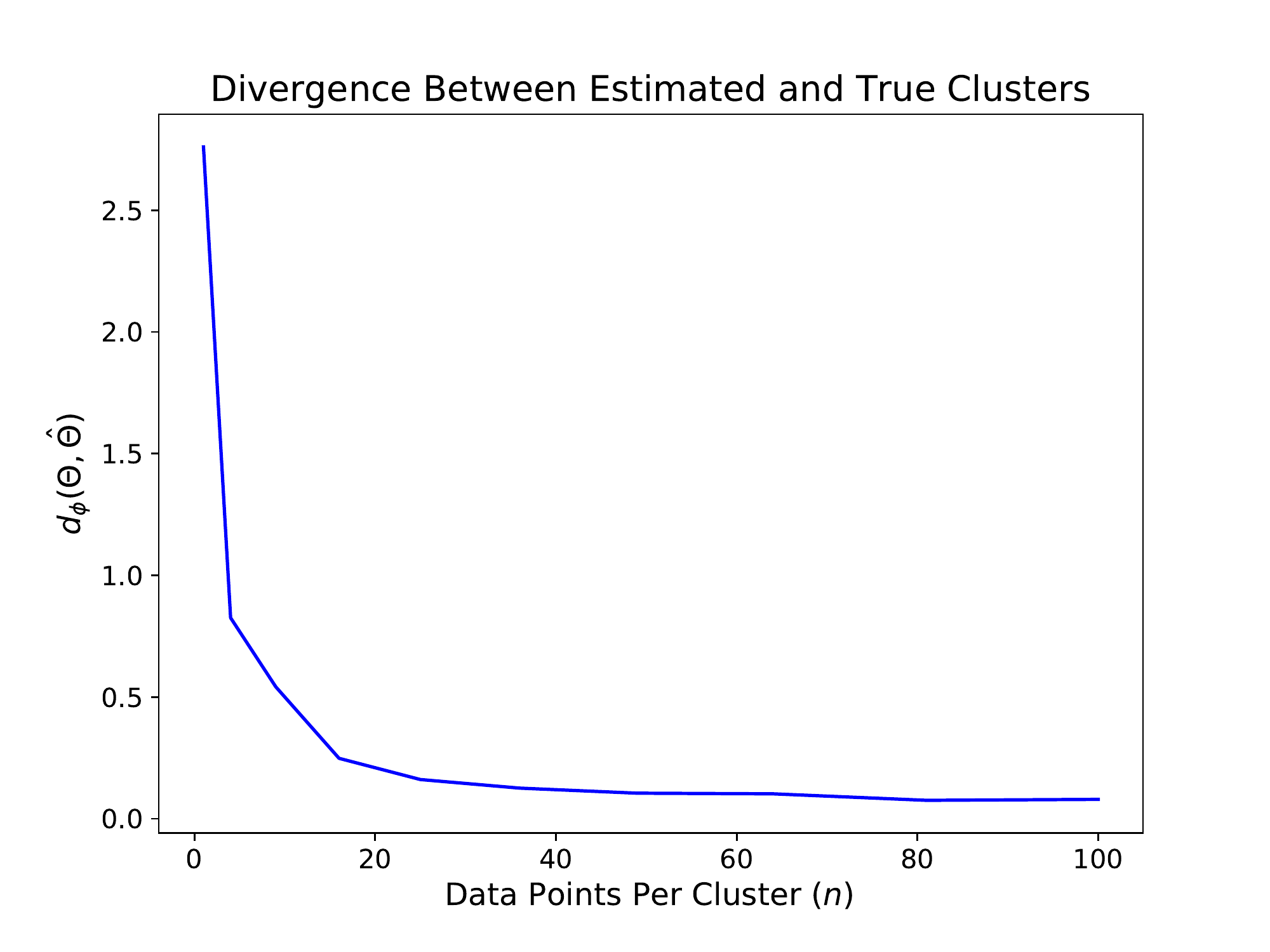}
    \vspace{-10pt}\caption{We see that the empirical convergence of Bregman power $k$-means to the true cluster centroids agrees with the $\mathcal{O}_P(n^{-1/2})$ convergence proposed in Theorem 3.8.} \label{fig:n-consistency}\vspace{-5pt}
\end{figure}

\textbf{Experiment 4:} As suggested by an anonymous reviewer, to check that empirical performance matches our theory that cluster centroids are $\sqrt{n}$-consistent (cf. Theorem 3.8), we include another experiment using the Poisson data setting ($p = 5$). We sample from the same three cluster centroids that we do in Experiment 3, and consider the Bregman divergence $d_\phi$ between the cluster centroids estimated by Bregman Power $k$-means ($\hth$) and the true cluster centroids ($\bTheta$) as the number of data points per cluster, $n$, varies from 1 to 100. For each $n$, we plot the lowest $d_\phi(\bTheta, \hth)$ across 100 random trials to decrease how likely it is that we report the divergence at a local optimum (as the theoretical result pertains to the true minimizers of the objective). The results are plotted in Figure \ref{fig:n-consistency}.

\begin{figure}[h] \vspace{-8pt}
   \hspace{-15pt}
    \includegraphics[width=9.6cm]{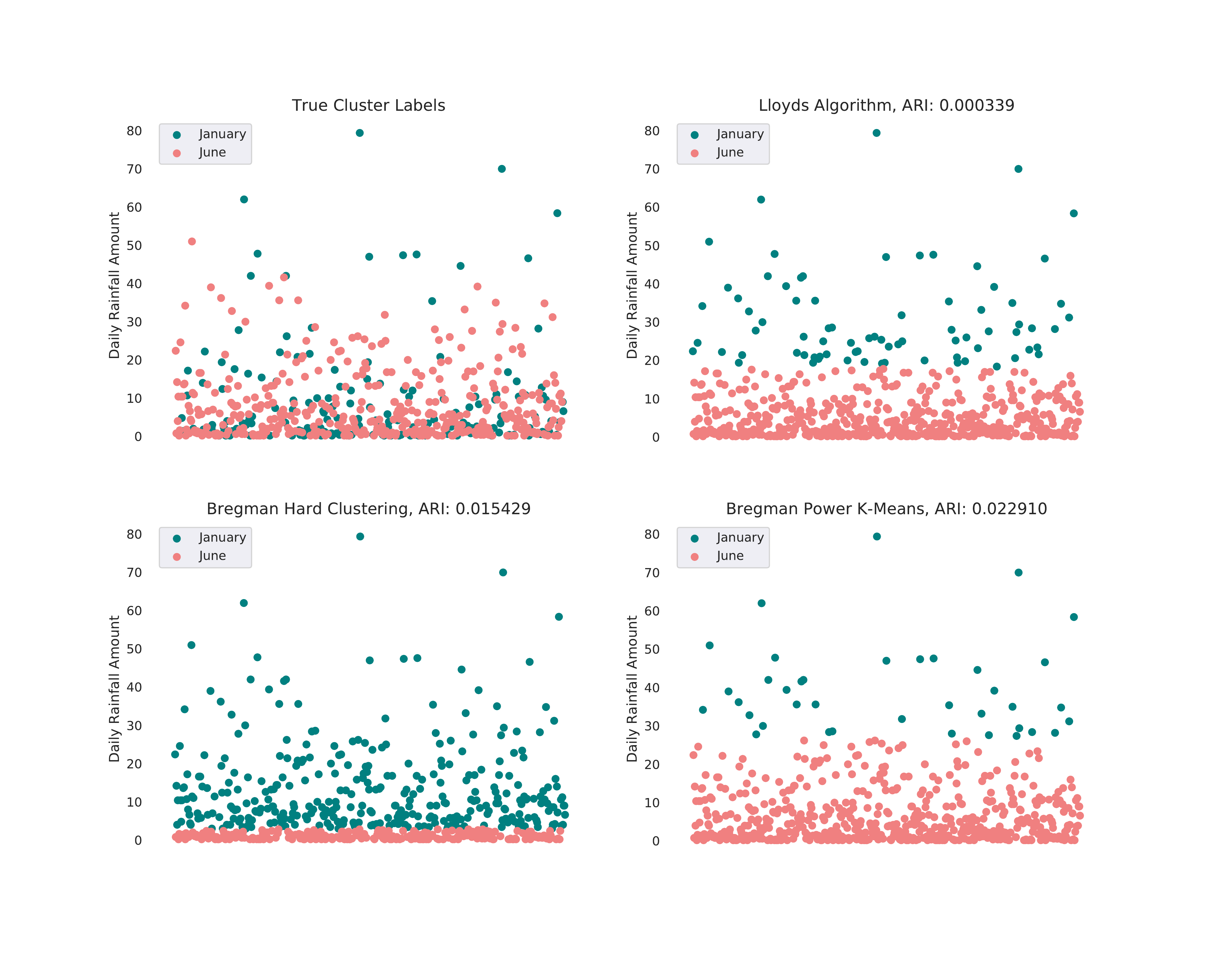}
    \vspace{-35pt}\caption{A visual comparison of clustering solutions under Lloyd's algorithm for $k$-means, Bregman hard clustering, and the proposed method on rainfall dataset. We see the ground truth features significant overlap between the groups, yet Bregman power $k$-means is able to find a closer partition than its peer methods.} \label{fig:visual_rainfall}\vspace{-2pt}
\end{figure}

\textbf{Rainfall data:} We now turn to a comparison of peer methods on a real dataset. Following recent Bregman clustering work of \citet{brecheteau2021robust}, we consider clustering rainfall data under a Gamma model with shape parameter $\alpha = 4$. Here we consider data from the Italian region of San Martino di Castrozza, collected across the years $1970-1990$ during the months of January ($177$ points) and June ($397$ points)\footnote{Publicly available at \url{https://cran.r-project.org/web/packages/hydroTSM/vignettes/hydroTSM_Vignette-knitr.pdf}}. Only days with non-zero rainfall amounts are included, and their values are often modeled by domain experts according to a Gamma distribution  \cite{CoeSternRainfall}. For Power $k$-means and Bregman Power $k$-means, an $s_0$ value of $-3.0$ is used. % were used to separate the data from January and June into separate clusters. An $s_0$ value of $-3.0$ was used for both Power $k$-means approaches. 

Similarly to the results in \citep{brecheteau2021robust}, we observe very low separability between clusters likely caused by the large number of days with very little rainfall as evident in the ground truth plotted in Figure \ref{fig:visual_rainfall}. Absolute ARIs of all methods in turn are quite low, yet Bregman power $k$-means offers the best performance, with a $\mathbf{48\%}$ improvement relative to the next best contender, Bregman hard clustering. Furthermore, Bregman power $k$-means has a runtime of 0.017 sec in comparison to the 0.975 sec runtime of Bregman hard clustering, offering a significant speedup just as we observed in Experiment 3. Both Bregman clustering methods perform an order of magnitude better than Lloyd's algorithm as well as power $k$-means. This is not unexpected given the highly skewed shapes of the data that bear little similarity to spherical clusters such as those arising under a Gaussian assumption. A visual comparison of the solutions can be found in Figure \ref{fig:visual_rainfall}.

% \begin{figure}[h] \vspace{-5pt}
%   \hspace{-15pt}
%     \includegraphics[width=9.6cm]{rainfall_data_vis_2.pdf}
%     \vspace{-30pt}\caption{A  comparison of clustering solutions for Rainfall Data from San Martino di Castrozza, Italy} \label{fig:visual_rainfall}
% \end{figure}
% ARIs for rainfall data
% lloyd's: 0.000339
% hard Bregman: 0.015429
% power $k$-means: 0.001682
% Bregman power $k$-means: 0.022910

\section{Discussion}
In this article, we have demonstrated the merits of adapting the recent power $k$-means method to Bregman divergences both through several novel theoretical contributions and a thorough empirical study. In line with what we would expect from the theory, we find that across a number of exponential family data generating mechanisms and as cluster numbers and dimension increase, Bregman power $k$-means consistently outperforms its Euclidean counterparts. On the other hand, we observe a marked improvement over Bregman hard clustering without annealing, showing that the majorization-minimization scheme proposed here successfully evades poor local minima. By tailoring our algorithm to the case of Bregman divergences within the family of power means, we preserve the simplicity of Lloyd's classical method, in contrast to generic gradient-based approaches put forth for a general robust clustering context that can be applied. 

%In terms of novel theory contributions, there are several findings worth emphasizing. \textcolor{blue}{[Insert here]}

A number of immediate extensions are worth exploring. The problem of feature selection alongside clustering is worthy of investigation, as discovering relevant variables may be crucial in high dimensional settings with low signal-to-noise ratio \citep{chakraborty2020entropy}, and may be an interpretable goal in its own right. Another useful generalization involves clustering matrix-variate data. The notion of Bregman divergence readily applies to matrices: 
\[ d_\phi(\bV, \bU) =  \phi(\bV) - \phi(\bU)- \langle \nabla\phi(\bU), \bV - \bU \rangle \]
where $\langle \bV, \bU \rangle = \text{Tr}(\bV \bU^T)$ denotes the inner product. %For instance, the squared Frobenius distance between $\bV, \bU$ is generated by the choice of $\phi(\bV) = \frac{1}{2} \lVert \bV \rVert_F^2$.

In exploring penalized or robust versions where the objective may no longer admit closed form updates, it is worth leveraging the information geometry behind Bregman divergences to design effective iterative extensions of our proposed method. One can benefit from second-order rate behavior within first-order schemes---as an example,  \citet{raskutti2015information} show that mirror descent performs \textit{natural} gradient descent along the dual Riemannian manifold under a Bregman proximity term. For exponential families, the Riemannian metric of the parameter space coincides with the Fisher information, a property that has proven useful in a number of machine learning applications \citep{hoffman2013stochastic}. 
Moreover, the explicit connection to exponential family likelihoods suggests it is natural to explore Bayesian approaches that may leverage the geometry of Bregman divergences \citep{ahn2020efficient,duan2021bayesian}, especially toward model-based clustering on exponential family mixtures.

From an optimization perspective, this work falls within the unified continuous optimization perspective set forth by \citet{teboulle2007unified}. There the authors suggested future extensions to various proximity measures, as well as investigation into the notion of ``best" smoothing functions for given formulations. Our article contributes significant progress on the former, while our novel rate analyses which provide results dependent on $\tth$  suggest a new theoretical tool toward formally tackling the latter. Within the power means framework, this open problem is also augmented by the question of optimal annealing schedules for decreasing parameters such as $s$ within a family of smoothing functions. We invite readers to consider these questions and  possible extensions in the many exciting directions for future work.

%\bibliography{example_paper}
%\bibliographystyle{icml2022}

%%%%%%%%%%%%%%%%%%%%%%%%%%%%%%%%%%%%%%%%%%%%%%%%%%%%%%%%%%%%%%%%%%%%%%%%%%%%%%%
%%%%%%%%%%%%%%%%%%%%%%%%%%%%%%%%%%%%%%%%%%%%%%%%%%%%%%%%%%%%%%%%%%%%%%%%%%%%%%%
% APPENDIX
%%%%%%%%%%%%%%%%%%%%%%%%%%%%%%%%%%%%%%%%%%%%%%%%%%%%%%%%%%%%%%%%%%%%%%%%%%%%%%%
%%%%%%%%%%%%%%%%%%%%%%%%%%%%%%%%%%%%%%%%%%%%%%%%%%%%%%%%%%%%%%%%%%%%%%%%%%%%%%%
\newpage
\appendix
\onecolumn
\section{Proof of Optimization Results from Section \ref{problem statement}}
\label{optimization proofs}
\paragraph{Proof of Theorem \ref{o1}}
Let $J(\bx):=d_\phi(\bx,\btheta)$ and let $P_{\mathscr{C}}(\btheta)$ be the projection of $\btheta$ onto $\mathscr{C}$, i.e. $P_{\mathscr{C}}(\btheta) = \argmin_{\bx \in \C} d_\phi(\bx,\btheta)$. Since $P_{\mathscr{C}}(\btheta)$ minimizes $J(\cdot)$ over $\C$, there exists a subgradient $\bd \in \partial J(P_\C(\btheta))$ such that 
\begin{equation}
\label{aaa1}
\langle \bd, \bx-P_\C(\btheta)\rangle \ge 0 .    
\end{equation}

We note that $J(P_\C(\btheta))=\{\nabla\phi(P_\C(\btheta)) - \nabla \phi(\btheta)\}$. Thus,  from equation \eqref{aaa1}, 
\begin{equation}
    \langle \nabla\phi(P_\C(\btheta)) - \nabla \phi(\btheta), \bx-P_\C(\btheta)\rangle \ge 0 .
\end{equation}
We proceed by observing that 
\[
    d_\phi(\bx,\btheta) - d_{\phi}(\bx,P_\C(\btheta)) - d_\phi(P_\C(\btheta),\btheta) \, \, = \, \, \langle \nabla\phi(P_\C(\btheta)) - \nabla \phi(\btheta), \bx-P_\C(\btheta)\rangle \, \,\ge \, \, 0.
\]
Thus, for any $\bTheta = [\btheta_1:\btheta_2:\dots: \btheta_k]^\top \in \Real^{k \times p}$, and so
\begingroup
\allowdisplaybreaks
\begin{align*}
  &d_\phi(\bx,\btheta_j)\, \, \ge \, \, d_{\phi}(\bx,P_\C(\btheta_j))+d_\phi(P_\C(\btheta_j),\btheta_j) \, \, \ge \, \, d_{\phi}(\bx,P_\C(\btheta_j))\, \quad \forall\, j=1,\dots,k. \\ 
  \implies & M_s\left(d_\phi(\bx,P_\C(\btheta_1)),\dots,d_\phi(\bx,P_\C(\btheta_k))\right) \, \,\le \, \, M_s\left(d_\phi(\bx,\btheta_1),\dots,d_\phi(\bx,\btheta_k)\right)\\
  \implies & \sum_{i=1}^n M_s\left(d_\phi(\bX_i,P_\C(\btheta_1)),\dots,d_\phi(\bX_i,P_\C(\btheta_k))\right) \, \, \le \, \, \sum_{i=1}^n M_s\left(d_\phi(\bX_i,\btheta_1),\dots,d_\phi(\bX_i,\btheta_k)\right) \\
\implies & f_s(\bTheta^\prime ) \, \, \le \, \, f_s(\bTheta),
\end{align*}
\endgroup
where $\bTheta^\prime = [P_{\C}(\btheta_1):\dots: P_{\C}(\btheta_k)]^\top \in \Real^{k \times p}$. Thus, $\inf_{\bTheta} f_s(\bTheta) = \inf_{\bTheta \subset C} f_s(\bTheta)$., and the result follows.
\paragraph{Proof of Theorem \ref{o2}}
For any $\bTheta \subset \C$ and $s_m \downarrow -\infty$, by monotonicity of power means, $f_{s_m}(\bTheta) \downarrow f_{k-means} (\bTheta)$ monotonically. The result now immediately follows from Dini's theorem \cite{rudin1964principles} in analysis as $\C$ is compact. 
\section{Proof of Statistical Theory Results from Section \ref{stat theory}}
Before we proceed to the proofs of main results of section \ref{stat theory}, we first state and prove the following two lemmas. The two results hold for any $M>0$.
\begin{lemma}\label{lemaa1}
For any $\btheta \in B(M)$, $\|\nabla \phi (\btheta)\|_2 \le \tau_2 M$.
\end{lemma}
\begin{proof}
For any $\btheta \in B(M)$,
\begin{align*}
    \|\nabla \phi (\btheta)\|_2 = \| \nabla \phi (\btheta) - \nabla \phi (\mathbf{0})\|_2 \le \tau_2 \|\btheta\|_2 \le \tau_2 M.
\end{align*}
\end{proof}
\begin{lemma}\label{lemaa2}
$\phi$ is $\tau_2 M $-Lipschitz on $B(M)$.
\end{lemma}
\begin{proof}
Appealing to the mean value theorem, 
\[\phi(\btheta) - \phi(\btheta^\prime) = \langle \nabla \phi(\xi) , \btheta - \btheta^\prime \rangle,\]
for some $\xi $ in the convex combinations of $\btheta$ and $\btheta^\prime$. Clearly, $\xi \in B(M)$. Thus, by Lemma \ref{lemaa1}, $\|\nabla \phi(\xi) \|_2 \le \tau_2 M$. Finally, appealing to Cauchy-Schwartz inequality, we get,
\[ |\phi(\btheta) - \phi(\btheta^\prime)| \le \|\nabla \phi(\xi)\|_2 \|\btheta - \btheta^\prime\|_2 \le \tau_2 M \|\btheta - \btheta^\prime\|_2 \]
\end{proof}

\paragraph{Proof of Theorem \ref{t1}}
\begin{proof}
Let $\boldsymbol{0}_{k \times p}$ be the $k \times p$ real matrix whose entries are all zero. By the definition of $\bTheta_n$, 
\[P_n \tilde{f}_{\hth} \le P_n \tilde{f}_{\boldsymbol{0}_{k \times p}} = \frac{1}{n} \sum_{i=1}^n \phi (\bX_i).\]
By Bernstein's inequality, we note that
\begin{align}
 & \sP\left( \frac{1}{n} \sum_{i=1}^n \phi (\bX_i) - P \phi > t\right) \le    \exp\left\{ - \frac{cn}{\sigma_\phi} \min\left\{\frac{t^2}{\sigma_\phi},t\right\}\right\} . \label{s1}
 %& \le \frac{\E_{\bX \sim P} \phi^2(\bX)}{n t^2} \label{s1}
\end{align}
 Upon taking $t = \sigma_\phi$, we observe that the RHS of \eqref{s1} reduces to $ e^{-cn}$. 
 
 Now, let $A_n = \{\bTheta \in \Real^{k \times p}: P_n \tilde{f}_{\bTheta} \le  P \phi + \sigma_\phi\}$. Then with probability at least $1 - e^{-cn}$, $\hth \in A_n$. Next, by definition of the set $A_n$, the inequality %$\bTheta \in A_n$, 
 \[P \tilde{f}_{\bTheta} = \E_{\{\bX\}_{i \in [n]} \overset{\text{i.i.d.}}{\sim P}} \left[ P_n \tilde{f}_{\bTheta}\right] \, \,\le\,\,   \E_{\{\bX\}_{i \in [n]} \overset{\text{i.i.d.}}{\sim P}}  \left[ P \phi + \sigma_\phi \right] =  P \phi + \sigma_\phi\]
  follows if $\bTheta \in A_n$. Thus, for all $n \in \mathbb{N}$, we have $A_n \subseteq \{\bTheta \in \Real^{k \times p}: P \tilde{f}_{\bTheta} \le  P \phi + \sigma_\phi\}$. 
 
 Lastly, consider the set $\mathcal{E} = \{\bTheta \in \Real^{k \times p}: P \tilde{f}_{\bTheta} \le  P \phi + \sigma_\phi\}$. For any $\bTheta \in \mathcal{E}$,
 \[\mathfrak{R}(\bTheta) = P \tilde{f}_{\bTheta} - P \tilde{f}_{\tth} \le P \tilde{f}_{\bTheta} \le  P \phi + \sigma_\phi.\]
 By assumption A~\ref{ass3}, $\text{dist}(\bTheta,\tth) \le M_{P \phi + \sigma_\phi} =\xi_P$. It is easy to see that $\text{dist}(\hth,\tth) \le \xi_P \implies \|\hat{\btheta}^{(n)}_j\|_2 \le \xi_P + \|\tth\|_F$. Thus, $\hth \subset B(\xi_P + \|\tth\|_F)$, with probability at least $1 - e^{-cn}$.
\end{proof}

Before we prove Theorem \ref{entropy}, we state and prove the following lemma, which asserts the Lipschitzness of $\tilde{f}_{\bTheta}$.
\begin{lemma}\label{lip}
Suppose A~\ref{ass2} holds. Then, for any $\bTheta, \bTheta^\prime \subset B(\xi_P + \|\tth\|_F)$ and $\bx \in \Real^p$, \[|\tilde{f}_{\bTheta}(\bx) - \tilde{f}_{\bTheta^\prime}(\bx)| \le k^{-1/s} \tau_2 ( 3 \xi_P + 3 \|\tth\|_F + \|\bx\|_2 ) \sum_{j = 1}^k  \|\btheta_j - \btheta_j^\prime\|_2.\]
\end{lemma}
\begin{proof}
For any $\bx \in \Real^p$,
\begin{align}
    |\tilde{f}_{\bTheta}(\bx) - \tilde{f}_{\bTheta^\prime}(\bx)| = &  \, \, \, k^{-1/s} \sum_{j = 1}^k |d_\phi(\bx,\btheta_j) - d_\phi(\bx,\btheta_j^\prime)| \label{eqq1}\\
       \le & \, \,  k^{-1/s} \sum_{j = 1}^k \big[|\phi(\btheta^\prime_j) - \phi(\btheta_j)| + \|\nabla \phi(\btheta^\prime_j) - \nabla \phi(\btheta_j)\|_2 (\|\bx\|_2 +\xi_P + \|\tth\|_F ) \nonumber   \\ 
       & \qquad + \|\nabla \phi (\btheta_j)\|_2 \|\btheta_j - \btheta_j^\prime\| \big] \nonumber\\ 
      \le & \, \, k^{-1/s} \sum_{j = 1}^k \big[ \tau_2 (\xi_P + \|\tth\|_F) + \tau_2 (\|\bx\|_2 +\xi_P + \|\tth\|_F ) + \tau_2 (\xi_P + \|\tth\|_F)\big] \|\btheta_j - \btheta_j^\prime\|_2 \label{eqq2}\\
      \le & \, \, k^{-1/s} \tau_2 ( 3 \xi_P + 3 \|\tth\|_F + \|\bx\|_2 ) \sum_{j = 1}^k  \|\btheta_j - \btheta_j^\prime\|_2 . \nonumber
\end{align}
Here, equation \eqref{eqq1} follows from \cite{beliakov2010lipschitz} and inequality \eqref{eqq2} follows from Lemmas \ref{lemaa1} and \ref{lemaa2}.
\end{proof}

% \begin{theorem}
% \[\mathcal{N}(\delta;\mathcal{F}, d_{2n} ) \le \left(\max\left\{1,\left\lfloor \frac{(\xi_P + \|\tth\|_F) C^{1/2}}{\delta}\right\rfloor\right\}\right)^{kp}, \]
% where $$C = 2 k^{2-2/s} \tau_2^2   p  n^{-1}\sum_{i = 1}^n ( 18 \xi_P^2  + 18  \|\tth\|_F^2 +  \|\bX_i\|_2^2 ).$$
% \end{theorem}
\paragraph{Proof of Theorem \ref{entropy}}
\begin{proof}
We begin by noting that if $|\theta_{ij} - \theta^\prime_{ij}| \le \rho$, then for all $i,j \in [n]$, 
\begin{align}
     d_{2n}^2 (\tilde{f}_{\bTheta},\tilde{f}_{\bTheta^\prime}) & = \frac{1}{n} \sum_{i=1}^n (\tilde{f}_{\bTheta}(\bX_i) - \tilde{f}_{\bTheta^\prime}(\bX_i))^2 \nonumber \\
    & \le k^{-2/s} \tau_2^2 (\sum_{j = 1}^k  \|\btheta_j - \btheta_j^\prime\|_2)^2 \frac{1}{n}\sum_{i = 1}^n ( 3 \xi_P + 3 \|\tth\|_F + \|\bX_i\|_2 )^2 \label{eq1}\\
    &  \le 2 k^{-2/s} \tau_2^2  (\sum_{j = 1}^k  \|\btheta_j - \btheta_j^\prime\|_2)^2 \frac{1}{n}\sum_{i = 1}^n ( 9 (\xi_P + \|\tth\|_F)^2 +  \|\bX_i\|_2^2 ) \nonumber\\
    &  \le 2 k^{2-2/s} \tau_2^2   p \rho^2 \frac{1}{n}\sum_{i = 1}^n ( 18 \xi_P^2  + 18  \|\tth\|_F^2 +  \|\bX_i\|_2^2 ). \label{s3}
\end{align}
Above, equation \eqref{eq1} follows from Lemma \ref{lip}. We take $\rho = \delta^{-1}\big( 2 k^{2-2/s} \tau_2^2   p  n^{-1}\sum_{i = 1}^n ( 18 \xi_P^2  + 18  \|\tth\|_F^2 +  \|\bX_i\|_2^2 ) \big)^{-1/2} $, which simplifies the RHS of \eqref{s3} to equal $\delta^2$.
Next, let
\[
\mathfrak{N}_\rho = \begin{cases}
\{-(\xi_P + \|\tth\|_F) + \frac{i}{2\rho }: 1 \le i \le \lfloor \frac{\xi_P + \|\tth\|_F}{\rho} \rfloor\} & \text{ if } \rho < \xi_P + \|\tth\|_F\\
\{\mathbf{0}\} & \text{ Otherwise.}
\end{cases}
\]
We consider a $(k\rho \sqrt{p})$-net of $B(\xi_P + \|\tth\|_F)$,
$\mathfrak{M}_\rho = \{\bTheta \in \Real^{k \times p}: \theta_{ij} \in \mathfrak{N}_\rho \}.$ From \eqref{s3}, $\mathfrak{F}_\rho = \{\tilde{f}_{\bTheta}: \bTheta \in \mathfrak{M}_\rho\}$ constitutes a $\delta$-cover of $\mathcal{F}$. Thus,
\[\mathcal{N}(\delta;\mathcal{F}, d_{2n} ) \le |\mathfrak{F}_\rho| \le |\mathfrak{M}_\rho| \le  \left(\max\left\{1,\left\lfloor \frac{\xi_P + \|\tth\|_F}{\rho}\right\rfloor\right\}\right)^{kp}.\]
Plugging in the value of $\rho$ gives us the desired result.
\end{proof}

\paragraph{Proof of Theorem \ref{rad}}
% \begin{theorem}
% \label{rad}
% \begin{align*}
%     \mathcal{R}_n(\mathcal{F}) \le 6 \tau_2  \sqrt{2 \pi     \sum_{i = 1}^n ( 18 \xi_P^2  + 18  \|\tth\|_F^2 + \E\|\bX\|_2^2 )}(\xi_P + \|\tth\|_F) \frac{k^{3/2 - 1/s}p}{\sqrt{n}}.
% \end{align*}
% \end{theorem}
\begin{proof}
It is easy to show using lemma \ref{lip} that
\begin{align*}
    \text{diam}(\F) & = \sup_{f,g \in \F} d_{2n}(f,g) \\
    & = \sup_{\bTheta \subset(\xi_P + \|\tth\|_F)} \sqrt{ 2 k^{-2/s} \tau_2^2  \big[\sum_{j = 1}^k  \|\btheta_j - \btheta_j^\prime\|_2\big]^2 \frac{1}{n}\sum_{i = 1}^n 9 (\xi_P + \|\tth\|_F)^2 +  \|\bX_i\|_2^2 } \\
    & \le 2 C^{1/2}(\xi_P + \|\tth\|_F).
\end{align*}
Now applying Dudley's chaining \cite{van1996weak},
\begin{align}
     \widehat{\mathcal{R}}_n(\mathcal{F}) &\le \frac{12}{\sqrt{n}} \int_0^{\text{diam}(\mathcal{F})} \sqrt{\log \mathcal{N}(\delta; \mathcal{F},d_{2n})} d\delta \nonumber\\
    & \le \frac{12\sqrt{kp}}{\sqrt{n}} \int_0^{2 C^{1/2}(\xi_P + \|\tth\|_F)} \sqrt{\log  \left(\max\left\{1,\left\lfloor \frac{M_\epsilon C^{1/2}}{\delta}\right\rfloor\right\}\right)} d\delta \nonumber\\
    & = \frac{12\sqrt{kp}}{\sqrt{n}}  C^{1/2}(\xi_P + \|\tth\|_F) \Gamma(3/2) \nonumber\\
    & = 6 \sqrt{\pi C}M_\epsilon \sqrt{\frac{kp}{n}}   \nonumber\\
    & = 6 \tau_2 \sqrt{2 \pi     \sum_{i = 1}^n \big( 18 \xi_P^2  + 18  \|\tth\|_F^2 + n^{-1}\|\bX_i\|_2^2 \big)}\,(\xi_P + \|\tth\|_F) \frac{k^{3/2 - 1/s}p}{\sqrt{n}}   \nonumber
\end{align}
Finally,
\begin{align*}
\mathcal{R}_{n}(\F) & = \E \widehat{\mathcal{R}}_n(\mathcal{F}) \\
& \le  6 \tau_2 \E \sqrt{2 \pi     \sum_{i = 1}^n \big( 18 \xi_P^2  + 18  \|\tth\|_F^2 + n^{-1}\|\bX_i\|_2^2 \big)}(\xi_P + \|\tth\|_F) \frac{k^{3/2 - 1/s}p}{\sqrt{n}}    \\
& \le 6 \tau_2  \sqrt{2 \pi     \sum_{i = 1}^n \big( 18 \xi_P^2  + 18  \|\tth\|_F^2 + \E\|\bX\|_2^2 \big)}(\xi_P + \|\tth\|_F) \frac{k^{3/2 - 1/s}p}{\sqrt{n}}.
\end{align*}

\end{proof}

\paragraph{Proof of Theorem \ref{concentration}}
% \begin{theorem}
% If $n \ge \log(2/\delta) \ge \frac{1}{2}$,
% \begin{align*}
%   & \sup_{\bTheta \subset B(\xi_P + \|\tth\|_F)} |P_n \tilde{f}_{\bTheta} - P \tilde{f}_{\bTheta}| \\
%   & \le 12 \tau_2  \sqrt{2 \pi     \sum_{i = 1}^n ( 18 \xi_P^2  + 18  \|\tth\|_F^2 + \E\|\bX\|_2^2 )}(\xi_P + \|\tth\|_F) \frac{k^{3/2 - 1/s}p}{\sqrt{n}} \\
%  & +16 e  k^{1-1/s}( \sigma_\phi+ \sigma \tau_2 (\xi_P + \|\tth\|_F) \|  ) \sqrt{\frac{2 \log (2/\delta)}{n}}.
% \end{align*}
% \end{theorem}
\begin{proof}
Let $g((\bX_i)_{i \in [n]}) = \sup_{\bTheta \subset B(\xi_P + \|\tth\|_F)} P_n \tilde{f}_{\bTheta} - P \tilde{f}_{\bTheta}$. It follows that
\[g((\bX_i)_{i \in [n]}) = \left[g((\bX_i)_{i \in [n]}) - \E g((\bX_i)_{i \in [n]}) \right]+ \E g((\bX_i)_{i \in [n]}).\]

Following %\citeauthor{maurer2021concentration} 
\citet{maurer2021concentration}, we bound the first term through concentration inequalities and the second term through symmetrization. It is easy to see that
\begin{equation}
    \label{a1}
    \E g((\bX_i)_{i \in [n]}) \le 2 \mathcal{R}_n(\F)
\end{equation}
Next, consider
\[\mathcal{B} = \{h: \F \to \Real: \sup_{f \in \F} |h(f)| < \infty \}.\]
Clearly, $\mathcal{B}$ is a normed vector space with $\|h\|_{\mathcal{B}} = \sup_{f \in \F} |h(f)|$. For any $\bX_i$, we define a corresponding $Y_i \in \mathcal{B}$ by 
\[Y_i(f) = \frac{1}{n} \left(f(\bX_i) - \E f(\bX_i)\right) \quad f \in \F.\]
Therefore, $\E [Y_i] \equiv 0$, and $g\left((\bX_i)_{i \in [n]}\right) = \|\sum_{i=1}^n Y_i\|_{\mathcal{B}}$. We now note that
\begingroup
\allowdisplaybreaks
\begin{align}
    &\|\|Y_i\|_{\mathcal{B}}\|_{\psi_1} \nonumber\\
    & = \frac{1}{n} \|\sup_{f \in \F} \left( f(\bX_i) - \E f(\bX_i)\right)\|_{\psi_1} \nonumber\\
    & = \frac{1}{n} \|\sup_{f \in \F} \E \left( f(\bX_i) - \E f(\bX_i^\prime)|(\bX_i)_{i \in [n]} \right)\|_{\psi_1} \nonumber\\
    & = \frac{1}{n} \|\sup_{\bTheta \subset B(\xi_P + \|\tth\|_F)} \E \left( \tilde{f}_{\bTheta}(\bX_i) -  \tilde{f}_{\bTheta}(\bX_i^\prime)|(\bX_i)_{i \in [n]} \right)\|_{\psi_1} \nonumber\\
    & \le  \frac{k^{-1/s}}{n} \|\sup_{\bTheta \subset B(\xi_P + \|\tth\|_F)} \sum_{j=1}^k \E \left(d_{\phi}(\bX_i,\btheta_j) -  d_{\phi}(\bX_i^\prime,\btheta_j)|(\bX_i)_{i \in [n]} \right)\|_{\psi_1} \nonumber\\
     & =  \frac{k^{-1/s}}{n} \|\sup_{\bTheta \subset B(\xi_P + \|\tth\|_F)} \sum_{j=1}^k \E \big((\phi(\bX_i) - \phi(\bX_i^\prime))  - \langle \nabla \phi(\btheta_j) , \bX_i - \bX_i^\prime \rangle |(\bX_i)_{i \in [n]} \big)\|_{\psi_1} \nonumber\\
     & \le    \frac{k^{1-1/s}}{n} \| \E \left(|(\phi(\bX_i) - \phi(\bX_i^\prime)| \big|(\bX_i)_{i \in [n]}\right)\|_{\psi_1} + \frac{k^{-1/s}}{n}\|\sup_{\bTheta \subset B(\xi_P + \|\tth\|_F)} \sum_{j=1}^k |\langle \nabla \phi(\btheta_j) , \bX_i - \bX_i^\prime \rangle |\big|(\bX_i)_{i \in [n]} \big)\|_{\psi_1} \nonumber\\
     & \le    \frac{k^{1-1/s}}{n} \| (\phi(\bX_i) - \phi(\bX_i^\prime) \|_{\psi_1} + \frac{k^{-1/s}}{n}\|\sup_{\bTheta \subset B(\xi_P + \|\tth\|_F)} \E \sum_{j=1}^k \| \nabla \phi(\btheta_j)\|_2  \|\bX_i - \bX_i^\prime\|_2\big|(\bX_i)_{i \in [n]} \big)\|_{\psi_1} \label{a2}  \\
     & \le    \frac{2 k^{1-1/s}}{n} \| \phi(\bX)  \|_{\psi_1} \nonumber + \frac{k^{1-1/s} \tau_2 (\xi_P + \|\tth\|_F)}{n}\|\sup_{\bTheta \subset B(\xi_P + \|\tth\|_F)}  \E \|\bX_i - \bX_i^\prime\|_2\big|(\bX_i)_{i \in [n]} \big)\|_{\psi_1} \nonumber\\
     & \le    \frac{2 k^{1-1/s}}{n} \| \phi(\bX)  \|_{\psi_1} + \frac{2 k^{1-1/s} \tau_2 (\xi_P + \|\tth\|_F)}{n}\|   \|\bX\|_2\|_{\psi_1} \nonumber\\
     & = \frac{2 k^{1-1/s}}{n} (\| \phi(\bX)  \|_{\psi_1} +  \tau_2 (\xi_P + \|\tth\|_F) \|   \|\bX\|_2\|_{\psi_1}) \nonumber.
\end{align}
\endgroup
Inequality \ref{a2} follows from Lemma 6 of \cite{maurer2021concentration}. Thus, by Proposition $7(ii)$ of \cite{maurer2021concentration}, 
\begingroup
\allowdisplaybreaks
\begin{align}
     g((\bX_i)_{i \in [n]}) - \E g((\bX_i)_{i \in [n]})  
    & \, \, \le \, \, 8 e \|g((\bX_i)_{i \in [n]})\|_{\psi_1} \sqrt{\frac{2 \log (1/\delta)}{n}} \nonumber\\
    & \le \, \, 8 e \|\|\sum_{i=1}^n Y_i\|_{\mathcal{B}}\|_{\psi_1} \sqrt{\frac{2 \log (1/\delta)}{n}} \nonumber\\
    & \le \, \,  8 e \sum_{i=1}^n \| \| Y_i\|_{\mathcal{B}}\|_{\psi_1} \sqrt{\frac{2 \log (1/\delta)}{n}} \nonumber\\
    & \le \, \, 16 e  k^{1-1/s}(\| \phi(\bX)  \|_{\psi_1} +  \tau_2 (\xi_P + \|\tth\|_F) \|   \|\bX\|_2\|_{\psi_1}) \sqrt{\frac{2 \log (1/\delta)}{n}}. \nonumber\\
\end{align}
Recalling the definition $$g((\bX_i)_{i \in [n]}) = \sup_{\bTheta \subset B(\xi_P + \|\tth\|_F)} P_n \tilde{f}_{\bTheta} - P \tilde{f}_{\bTheta},$$ with probability at least $1-\delta$,
\begin{align*}&  g((\bX_i)_{i \in [n]})
   \le 2 \mathcal{R}_n(\F)  +16 e  k^{1-1/s}(\| \phi(\bX)  \|_{\psi_1} +  \tau_2 (\xi_P + \|\tth\|_F) \|   \|\bX\|_2\|_{\psi_1}) \sqrt{\frac{2 \log (1/\delta)}{n}}.
\end{align*}
\endgroup
%\textcolor{red}{Check this proof and the inequalities, and be consistent with when you abbreviate the supremum expression to g} 
Similarly, we can show that with probability at least $1-\delta$,
\begin{align*}
    & \sup_{\bTheta \subset B(\xi_P + \|\tth\|_F)} P \tilde{f}_{\bTheta} - P_n \tilde{f}_{\bTheta} \le 2 \mathcal{R}_n(\F)  +16 e  k^{1-1/s}(\| \phi(\bX)  \|_{\psi_1} +  \tau_2 (\xi_P + \|\tth\|_F) \|   \|\bX\|_2\|_{\psi_1}) \sqrt{\frac{2 \log (1/\delta)}{n}}.
\end{align*}
Combining the above and replacing $\delta$ by $\delta/2$, we get that with probability at least $1-\delta$,
\begin{align*}
    & \sup_{\bTheta \subset B(\xi_P + \|\tth\|_F)} |P_n \tilde{f}_{\bTheta} - P \tilde{f}_{\bTheta}| \le \mathcal{R}_n(\F)  +16 e  k^{1-1/s}(\| \phi(\bX)  \|_{\psi_1} +  \tau_2 (\xi_P + \|\tth\|_F) \|   \|\bX\|_2\|_{\psi_1}) \sqrt{\frac{2 \log (2/\delta)}{n}}.
\end{align*}
Now by bounding $\mathcal{R}_n(F)$ using Theorem \ref{rad}, we obtain our desired result.
\end{proof}

\section{Proof of Theorem \ref{er}}
% \begin{theorem}
% If $n \ge \log(2/\delta) \ge \frac{1}{2}$, with probability at least $1 - 2e^{-n} - e^{-cn} - \delta$, 
% \begin{align*}
%      \mathfrak{R}(\hth)
%     & \, \, \le \, \,  24 \tau_2  \sqrt{2 \pi     \sum_{i = 1}^n ( 18 \xi_P^2  + 18  \|\tth\|_F^2 + \E\|\bX\|_2^2 )}(\xi_P + \|\tth\|_F) \frac{k^{3/2 - 1/s}p}{\sqrt{n}} \nonumber\\
%  & \qquad + 32 e  k^{1-1/s}( \sigma_\phi+ \sigma \tau_2 (\xi_P + \|\tth\|_F) \|  ) \sqrt{\frac{2 \log (2/\delta)}{n}}
% \end{align*}
% \end{theorem}
\begin{proof}
We note the following:
\begin{align}
     \mathfrak{R}(\hth) & = P \tilde{f}_{\hth} - P \tilde{f}_{\tth} \nonumber\\
    & \le P \tilde{f}_{\hth} - P_n \tilde{f}_{\hth} + P_n \tilde{f}_{\hth}- P \tilde{f}_{\tth} \nonumber\\
    & \le P \tilde{f}_{\hth} - P_n \tilde{f}_{\hth} + P_n \tilde{f}_{\tth}- P \tilde{f}_{\tth} \nonumber\\
    & \le 2 \sup_{\hth \subset B(\xi_P + \|\tth\|_F)} |P_n \tilde{f}_{\bTheta} - P f_{\bTheta}|  \label{r1}\\
    & \le 24 \tau_2  \sqrt{2 \pi     \sum_{i = 1}^n \big( 18 \xi_P^2  + 18  \|\tth\|_F^2 + \E\|\bX\|_2^2 \big)}(\xi_P + \|\tth\|_F) \frac{k^{3/2 - 1/s}p}{\sqrt{n}} \nonumber\\
 & + 32 e  k^{1-1/s}\left[ \sigma_\phi+ \sigma \tau_2 (\xi_P + \|\tth\|_F) \|  \right] \sqrt{\frac{2 \log (2/\delta)}{n}} \label{r2}.
\end{align}
Here, \eqref{r1} holds with probability at least $1 - e^{-cn}$ by Theorem \ref{t1} and \eqref{r2} holds with probability at least $1-\delta$ by a simple application of Theorem \ref{concentration}.
\end{proof}

\paragraph{Proof of Theorem \ref{consistency}}
\begin{proof}
For simplicity of notations, let $C_1 = \max \left\{24 \tau_2  C^\prime (\xi_P + \|\tth\|_F) k^{3/2 - 1/s}p, 32 e \sigma \tau_2 k^{1-1/s}(1 + \xi_P + \|\tth\|_F  )\right\}$. From Theorem \ref{er}, we know that with probability at least $1-\delta-e^{-cn}$, 
\begin{equation}\label{q4}
  \mathfrak{R}(\hth) \le \frac{C_1}{\sqrt{n}} + C_1 \sqrt{\frac{\log (2/\delta)}{n}}  .
\end{equation}
Now, fix $\epsilon>0$: if $n \ge \max\left\{\frac{4C_1^2}{\epsilon^2}, \frac{\epsilon^4}{4 C_1^4}\right\}$ and $\delta = 2 \exp\left(-\frac{\sqrt{n}\epsilon^2}{2C^2}  \right)$, the RHS of \eqref{q4} becomes no bigger than $\epsilon$. Thus,
\[\sP \left( |P\tilde{f}_{\hth} - P \tilde{f}_{\bTheta^\ast}| > \epsilon \right) \le 2 \exp\left(-\frac{\sqrt{n}\epsilon^2}{2C^2}  \right), \quad \forall \, n \ge 4C^2/\epsilon^2 .\]
Since the series $\sum_{n=1}^\infty \exp\left(-\frac{\sqrt{n}\epsilon^2}{2C^2}  \right) $ is convergent from the above equation, so is $\sP \left( |P\tilde{f}_{\hth} - P \tilde{f}_{\bTheta^\ast}| > \epsilon \right)$. Hence, $P\tilde{f}_{\hth} \xrightarrow{a.s.} P \tilde{f}_{\bTheta^\ast}$. Thus, for any $\epsilon>0$, it follows that $P \tilde{f}_{\hth} \le  p \tilde{f}_{\bTheta^\ast} + \epsilon$ almost surely with respect to $[P]$ for $n$ sufficiently large. From assumption A~\ref{ass3}, for any fixed $\eta>0$ and $n$ large, we also know %\ref{a5},
$\text{dist}(\hth,\bTheta^\ast) \le  \eta$ almost surely with respect to $[P]$. Together, $\text{dist}(\hth, \bTheta^\ast) \xrightarrow{a.s.}0$, which proves the result.% \textcolor{red}{there is no A4 i don't think}
\end{proof}

\section{Assumption of sub-Gaussianity of $\|\bX\|_2$}
\label{subg}
In this section, we show that a concrete sufficient condition for A~\ref{ass1} to be satisfied is that $\|\bX\|_2$ is sub-Gaussian. We state and prove this result in the following theorem.

\begin{theorem}
If $\|\bX\|_2$ is sub-Gaussian then under assumption A~\ref{ass2}, 
\begin{itemize}
    \item[(a)] $\|\|\bX\|_2\|_{\psi_1} < \infty$.
    \item[(b)] $\|\phi(\bX)\|_{\psi_1} < \infty$.
\end{itemize}
\end{theorem}
\begin{proof}
\textit{(a)}  We know that \( \|\|\bX\|_2\|_{\psi_2} \triangleq  \sup_{p \in \mathbb{N}}\frac{(\E \|\bX\|_2^p)^{1/p}}{\sqrt{p}} < \infty\) from Proposition 2.5.2 of \cite{vershynin2018high}.  Thus, 
\[\|\|\bX\|_2\|_{\psi_1} =  \sup_{p \in \mathbb{N}}\frac{(\E \|\bX\|_2^p)^{1/p}}{p} \le  \sup_{p \in \mathbb{N}}\frac{(\E \|\bX\|_2^p)^{1/p}}{\sqrt{p}} < \infty.\]

\textit{(b)}
We note that
\begin{align}
    \|\phi(\bX)\|_{\psi_1}& =  \sup_{p \in \mathbb{N}}\frac{(\E |\phi(\bX)|^p)^{1/p}}{p} \nonumber\\
    & = \sup_{p \in \mathbb{N}}\frac{(\E |d_\phi(\bX,\mathbf{0})|^p)^{1/p}}{p} \nonumber\\
    & \le \sup_{p \in \mathbb{N}}\frac{(\E \tau_2 \|\bX\|_2^{2p})^{1/p}}{p}\label{sg1}\\
    & = \tau_2 \sup_{p \in \mathbb{N}}\frac{(\E \|\bX\|_2^{2p})^{1/p}}{p} \nonumber\\
    & = \tau_2 \sup_{p \in \mathbb{N}}\left(\frac{(\E \|\bX\|_2^{2p})^{1/2p}}{\sqrt{p}}\right)^2 \nonumber\\
    & = 2 \tau_2  \left(\sup_{p \in \mathbb{N}}\frac{(\E \|\bX\|_2^{2p})^{1/2p}}{\sqrt{2p}}\right)^2 \nonumber\\
    & \le 2 \tau_2 \|\|\bX\|_2\|_{\psi_2}^2 < \infty. \nonumber
\end{align}
Inequality \eqref{sg1} follows from Lemma B.1 appearing in the Supplement of \citet{telgarsky2013moment}.
\end{proof}
%%%%%%%%%%%%%%%%%%%%%%%%%%%%%%%%%%%%%%%%%%%%%%%%%%%%%%%%%%%%%%%%%%%%%%%%%%%%%%%
%%%%%%%%%%%%%%%%%%%%%%%%%%%%%%%%%%%%%%%%%%%%%%%%%%%%%%%%%%%%%%%%%%%%%%%%%%%%%%%

\end{document}